\documentclass{article}

\usepackage{microtype}
\usepackage{graphicx}
\usepackage{subfigure}
\usepackage{booktabs} 

\usepackage[utf8]{inputenc}
\usepackage[T1]{fontenc}
\usepackage[english]{babel}
\usepackage{lmodern}
\usepackage{setspace}
\usepackage{amsmath}
\usepackage{amssymb}
\usepackage{dsfont}
\usepackage{amsthm}
\usepackage{mathrsfs}
\usepackage{braket}
\usepackage{bbold}
\usepackage{cancel}
\usepackage{units}
\usepackage{array}
\usepackage{todonotes}

\newcommand{\Fb}{\boldsymbol{F}}
\newcommand{\Gb}{\boldsymbol{G}}

\newcommand{\Tb}{\boldsymbol{T}}

\newcommand{\Wb}{\boldsymbol{W}}
\newcommand{\Xb}{\boldsymbol{X}}

\newcommand{\ub}{\boldsymbol{u}}
\newcommand{\vb}{\boldsymbol{v}}

\newcommand{\xb}{\boldsymbol{x}}
\newcommand{\yb}{\boldsymbol{y}}
\newcommand{\zb}{\boldsymbol{z}}

\newtheorem{theorem}{Theorem}[section]
\newtheorem{lemma}[theorem]{Lemma}
\newtheorem{proposition}[theorem]{Proposition}

\newtheorem{definition}{Definition}[section]

\DeclareMathOperator*{\argmax}{arg\,max}

\newcommand\independent{\protect\mathpalette{\protect\independenT}{\perp}}
\def\independenT#1#2{\mathrel{\rlap{$#1#2$}\mkern2mu{#1#2}}}

\usepackage{hyperref}

\usepackage[accepted]{icml2019}

\begin{document}

\twocolumn[
\icmltitle{Discovering Conditionally Salient Features with Statistical Guarantees}

\icmlsetsymbol{equal}{*}

\begin{icmlauthorlist}
\icmlauthor{Jaime Roquero Gimenez}{stats}
\icmlauthor{James Zou}{bds}
\end{icmlauthorlist}

\icmlaffiliation{stats}{Department of Statistics, Stanford University, Stanford USA}
\icmlaffiliation{bds}{Department of Biomedical Data Science, Stanford University, Stanford USA}

\icmlcorrespondingauthor{James Zou}{jamesz@stanford.edu}

\icmlkeywords{Feature selection, Knockoffs}

\vskip 0.3in
]

\printAffiliationsAndNotice{} 

\begin{abstract}
The goal of feature selection is to identify important features that are relevant to explain an outcome variable. Most of the work in this domain has focused on identifying \emph{globally} relevant features, which are features that are related to the outcome using evidence across the entire dataset. We study a more fine-grained statistical problem: \emph{conditional feature selection}, where a feature may be relevant depending on the values of the other features. For example in genetic association studies, variant $A$ could be associated with the phenotype in the entire dataset, but conditioned on variant $B$ being present it might be independent of the phenotype. In this sense, variant $A$ is globally relevant, but conditioned on $B$ it is  no longer locally relevant in that region of the feature space.  We present a generalization of the knockoff procedure that performs \emph{conditional feature selection} while controlling a generalization of the false discovery rate (FDR) to the conditional setting. By exploiting the feature/response model-free framework of the knockoffs, the quality of the statistical FDR guarantee is not degraded even when we perform conditional feature selections. We implement this method and present an algorithm that automatically partitions the feature space such that it enhances the differences between selected sets in different regions, and validate the statistical theoretical results with experiments.
\end{abstract}

\section{Introduction}
Given $d$ features $\Xb = (X_1, \dots, X_d)$ and a response variable $Y$, in many settings only a small subset of $X_i$ are relevant for $Y$. The goal of global statistical feature selection is to identify those relevant features, while ideally providing some statistical control on the rate of false discoveries. This problem has been extensively studied in the statistics and machine learning literature \cite{benjamini1995controlling, holm1979simple, benjamini2001control, efron2012large}. One scientific domain that requires extensive use of feature selection is genomics. There $Y$ corresponds to a phenotype of interest (e.g. whether a patient has a disease or not), and $\Xb$ contains information about individual mutations at genomic sites with high variability, or single nucleotide polymorphisms (SNPs). Researchers are interested to identify the small subset of features/SNPs that affect the disease risks based on information gathered from cohorts of participants with and without the phenotype of interest in genome-wide association studies (GWAS) \cite{hirschhorn2005genome, mccarthy2008genome}. 

We call this \emph{global} feature selection because its goal is to identify the relevant features for $Y$ regardless of their location in the feature space. GWAS looks for the subset of SNPs that best help to predict the phenotype, regardless of the actual values of the SNPs themselves. The purpose of this work is to extend the feature selection problem to a \emph{conditional} feature selection problem, where we want to identify relevant features conditionally on being on a sub-region of the feature space. One motivation for this comes again from genomics. Given a SNP (or a set of SNPs) that we believe is relevant for understanding a disease based on prior biological knowledge, or from a previous GWAS study, we want to run a GWAS conditionally on such SNP having a prescribed value. As an example, a few SNPs in a few loci have been identified as major potential causes breast cancer, such as the mutations in the BRCA1/2 genes \cite{o2010brca1}. We would like to run a GWAS conditioned on not having these variants to identify the other sources of genetic variation that are related to breast cancer.

There are two main advantages when doing \emph{conditional} feature selection. First, a feature that is selected through a global feature selection procedure may no longer be relevant conditioned on a local scale. In genomics, the presence or absence of a mutation can activate or deactivate different biological pathways. Therefore another mutation can be relevant or not for a given phenotype depending on whether its pathway was activated by the first mutation. Being able to distinguish between globally and locally relevant features is therefore essential for understanding the complex interactions between features and outcome. Also, by focusing on a local region, one could increase power to identify locally relevant features, which could potentially remain undiscovered if running a global feature selection procedure. 

Global feature selection procedures are unable to do this kind of conditional feature selection while preserving statistical guarantees \cite{sun2010local, tang2014feature}. Statistical models that allow for inference while conditioning on some value of the features are very limited, relying upon  strong assumptions such as simple linear models that are generally not appropriate for  high-dimensional settings of interest. A second approach for conditional feature selection is to restrict the initial dataset to points lying in the region of interest, so that the outcomes of the selection procedure are only related to such region. However, this approach substantially reduces  the amount of data available to fit the model, leading to more variance in the estimates. Furthermore, this undermines the validity of the assumptions upon which the construction of p-values relies, so that  statistical false discovery guarantees are even less likely to hold.

\paragraph{Our Contributions}
We develop a new method which extends recently developed knockoff procedure to perform conditional feature selection while still guaranteeing that the local false discovery rate (defined in Sec. 2) is controlled. By using the whole dataset to fit our model of the feature distribution $P^{\Xb}$ during the first step of the procedure, we make sure that the FDR control for our local knockoff method is as valid as in the usual global method. We extend the knockoff machinery so that we are able to localize the second step of the procedure, where we construct local importance scores for each feature. The selection sets are obtained conditionally on being in a subregion of the space. The main contributions of this paper are in laying out the new framework for conditional feature selection along with proposing a new knockoff algorithm with mathematical guarantees. We also validate the algorithm on experiments.

\section{Local Null Features and Local False Discovery Rate}

We refer to conditional feature selection the task of identifying a relevant subset of features that explain the response conditionally on being in some subspace of the feature space. As such, we need to define a notion of null feature conditionally on being on a subspace. We will refer to such generalization of null feature as \emph{local} null features; we will use \emph{local} and \emph{conditional} interchangeably to emphasize the conditioning on a local region of the feature space.

\paragraph{Preliminaries} Given a positive integer $d$, let $[d]= \{1,\dots,d\}$ and $\mathcal{P}([d])$ denote the collection of all subsets of $[d]$. For $x\in \mathbb{R}^d$, $r>0$, let $B(x,r)$ be the ball centered at $x$ of radius $r$ for the sup norm. For any two mappings $\alpha, \beta: \mathbb{R}^d \rightarrow \mathcal{P}([d])$, we write $\alpha \subset \beta$ (resp. $\alpha \setminus \beta$) if $\alpha(\xb) \subset \beta(\xb)$ (resp. $\alpha(\xb) \setminus \beta(\xb)$) for all $\xb \in \mathbb{R}^d$. We begin by introducing the usual setting of feature selection procedures. We consider the data $(\mathbb{X}, \mathbb{Y}) = (\Xb_i, Y_i )_{1\leq i \leq n} $ as a sequence of $n$ i.i.d. samples from some unknown joint distribution: $(\Xb_i, Y_i) = (X_{i1},\dots,X_{id},Y_i) \sim P^{\Xb Y}$, $i = 1,\dots,n$. We then define the set of null features $\mathcal{H}_0\subset [d]$ such that $j \in \mathcal{H}_0$ if $X_j \independent Y | \Xb_{-j}$ (where the $-j$ subscript indicates all variables except the $j$th, bold letters indicate vectors, and double-struck letters for the data matrix of stacked vectors). The set of non-null features $\mathcal{H}_1 = [d] \setminus \mathcal{H}_0$, also called alternatives, is important because these capture the truly influential effects and the goal of selection procedures is to identify them. We will denote by $(\Xb,Y)$ the random variables whenever we make probabilistic statements.

\paragraph{Local Null Features} We now consider a local notion of null feature, that is, we are interested in whether any given feature has an impact on the outcome conditionally on the feature set $\Xb$ being on the neighborhood of some point $\xb\in \mathbb{R}^d$. Indeed, the set of relevant features may depend on which region of the feature space we are in. We generalize the definition of a null feature as follows:

\begin{definition}
Define mappings
\begin{equation}
    \mathcal{H}_0^0 : \begin{cases}
    \mathbb{R}^d \rightarrow \mathcal{P}([d]) \\
    \xb \mapsto \mathcal{H}_0^0(\xb)
    \end{cases}
    \quad \text{and} \quad
    \mathcal{H}_1^0 = [d] \setminus \mathcal{H}_0^0
\end{equation}
such that $j \in \mathcal{H}_0^0(\xb)$ if $X_j \independent Y | \Xb_{-j} = \xb_{-j}$. We say that the $j$th feature is a local null at $\xb \in \mathbb{R}^d$ if $j \in \mathcal{H}_0^0(\xb)$. For $r>0$, define 
\begin{equation}
    \mathcal{H}_0^r : \begin{cases}
    \mathbb{R}^d \rightarrow \mathcal{P}([d]) \\
    \xb \mapsto \mathcal{H}_0^r(\xb) = \bigcap_{\zb \in B(\xb, r)}\mathcal{H}_0^0(\zb)
    \end{cases}
\end{equation}
We say that the $j$th feature is a $r$-local null at $\xb \in \mathbb{R}^d$ if $j \in \mathcal{H}_0^r(\xb)$. Finally, we define the set of $r$-local non-nulls at $\xb$ by $\mathcal{H}_1^r(\xb) = [d] \setminus \mathcal{H}_0^r(\xb)$. 
\end{definition}

A straightforward generalization of this definition consists in defining local nulls based on a general subset of $\mathbb{R}^d$, although we will keep this simplified version in what follows. From now on we will assume the joint distribution of $(\Xb, Y)$, has a positive density $p(\xb,y)>0$ with respect to a base product measure, so that we can write the conditional distribution of $Y|\Xb$ as $p(y|\xb)$. Saying $j$ is a (global) null (i.e. $X_j \independent Y |\Xb_{-j}$) is equivalent to $p(y|\xb)$ not depending on $x_j$. In contrast, $j$ is a local null at $\xb$ indicates that, if we fix the values of $\xb_{-j}$, then the expression $p(y|\xb) = p(y|x_j, \xb_{-j})$ as a function of $y$ and $x_j$ does not depend on $x_j$. An immediate consequence is that the set of local nulls does not vary if we move along the null features. Also, identifying local non-nulls at smaller $r$ implies having a more detailed information about the distribution $P^{\Xb Y}$. Our ability to recover the set of local non-nulls for smaller $r$ will be limited by the availability of data in the $r$-radius neighborhood of $\xb$. We present these results more formally as follows:

\begin{proposition}\label{prop:inclusion}
For any $\xb, \zb \in \mathbb{R}^d$, 
\[\xb_{[d]\setminus \mathcal{H}_0^0(\xb)} = \zb_{[d]\setminus  \mathcal{H}_0^0(\xb)} \quad \Rightarrow \quad  \mathcal{H}_0^0(\xb) =   \mathcal{H}_0^0(\zb)
\]
If $r'>r$, then $\mathcal{H}_0^{\infty} = \mathcal{H}_0 \subset \mathcal{H}_0^{r'} \subset \mathcal{H}_0^{r}$ the set of global nulls. Equivalently, $\mathcal{H}_1^{r} \subset \mathcal{H}_1^{r'}$. 
\end{proposition}

\begin{proof}
If $j$ is a local null at $\xb$, then $p(y|\xb)$ does not depend on $x_j$ when fixing $\xb_{-j}$. Therefore $j$ is still a local null whenever we change its value provided the other coordinates are fixed. The first assertion is a consequence of this, extended to the whole set of nulls. The second assertion is immediate following the definition $\mathcal{H}_0^r(\xb) = \bigcap_{\zb \in B(\xb, r)}\mathcal{H}_0^0(\zb)$.
\end{proof}

\paragraph{Local False Discovery Rate} The output of a global feature selection procedure is a data-dependent set of selected features \mbox{$\hat{\mathcal{S}}\subset \{1,\dots,d\}$}. One popular statistical criterion to evaluate the performance of the procedure is the False Discovery Rate (FDR), which stands for the expected value of the proportion of false discoveries: $\text{FDR} = \mathbb{E}\Big[ \frac{|\hat{\mathcal{S}}\cap \mathcal{H}_0|}{|\hat{\mathcal{S}}|} \Big]$. The ratio $\frac{|\hat{\mathcal{S}}\cap \mathcal{H}_0|}{|\hat{\mathcal{S}}|}$ is also called False Discovery Proportion (FDP). Given a prescribed target $q\in (0,1)$, we want to guarantee that our procedure's FDR is upper bounded by $q$. If we now switch to a local/conditional perspective, we are now looking for procedures that construct mappings $\xb \mapsto \hat{\mathcal{S}}(\xb)$ where we want $\hat{\mathcal{S}}(\xb)$ to match the set of local nulls at $\xb$. The criterion to evaluate these new selection mapping procedures should also be point-dependent. We can now accordingly extend the definition of FDP and FDR to the local setting.

\begin{definition}
For $r>0$, and a given feature selection procedure that for $\xb \in \mathbb{R}^d$ outputs a data-dependent set of selected features $\hat{\mathcal{S}}(\xb)$, we define the $r$-local False Discovery Proportion ($\text{FDP}^r$) and $r$-local False Discovery Rate ($\text{FDR}^r$) as the following functions of $\xb \in \mathbb{R}^d$:
\[
\text{FDP}^r\!(\xb)\! =\! \frac{|\hat{\mathcal{S}}(\xb)\cap \mathcal{H}_0^r(\xb)|}{1\vee | \hat{\mathcal{S}}(\xb)|} \quad \text{FDR}^r\!(\xb)\! =\! \mathbb{E}[\text{FDP}^r(\xb)]
\]
\end{definition}
The randomness in the $\text{FDP}^r$ expression comes from the fact that the construction of $\hat{\mathcal{S}}(\xb)$ depends on the random dataset of observations $(\mathbb{X},\mathbb{Y})$. The motivation for such criterion is the following: in a global feature selection procedure with FDR control, we want to construct $\hat{\mathcal{S}}$ as close as possible from $\mathcal{H}_1^{\infty}$, while penalizing whenever $j \in \hat{\mathcal{S}}\cap \mathcal{H}_0^{\infty}$. If we want to have a more granular approximation of the local non-nulls at $\xb$ by constructing $\hat{\mathcal{S}}(\xb)$, we need to penalize whenever $j \in \hat{\mathcal{S}}(\xb)\cap \mathcal{H}_0^r(\xb)$. Indeed, given that $\mathcal{H}_0^{\infty} \subset \mathcal{H}_0^{r}$ by Proposition~\ref{prop:inclusion}, if we were to construct local selection sets using the global FDR we wouldn't penalize the fact that in a neighborhood of $\xb$ a feature $j$ is locally null even though it is non-null at a far away part of the feature space. A new issue that arises from this statistical objective is that we no longer need to control for a single real value, but for a function $\mathbb{R}^d \rightarrow [0,1]$. From now on, we will focus on controlling point-wise $\text{FDR}^r$ at a finite number of distinct points far apart, in an independent manner. A future research direction will consist in defining a more comprehensive FDR objective across selection sets at several points in the space.

The strength of the local procedure that we now present is that the validity of the statistical guarantees on the local FDR is as strong as for the global procedure. The ability of making true local discoveries at a given point will be limited by the availability of data in the neighborhood of such point. However the rate of false positives will still be controlled.

\section{Controlling Local FDR Through Knockoffs with Local Importance Scores}

We now present the global knockoff procedure, and then define the generalization to a local setting through an extension of the importance scores. Our procedure requires the analyst to provide information on the choice of the partitioning of the feature space. We end this section by providing two algorithms based on heuristics to automate this choice.

\subsection{The Knockoff Procedure}

Traditionally feature selection is based on a parametric model we assume captures the relationship between $\Xb$ and $Y$. By fitting such model to the available data, with the appropriate statistical assumptions we can assess the probability of each parameter to be different from 0, and use it as a proxy for the likelihood that the corresponding feature is relevant for predicting $Y$ \cite{lippert2011fast, miller2002subset, tibshirani1996regression}. These assumptions on the distribution of $Y|\Xb$ are often inadequate, as the complexity of the distribution of the response given the features is seldom limited to a linear regression (or even to other more complex models such as generalized linear models). This model mismatch can break down the statistical guarantees of the procedure, or drastically reduce its ability to find relevant features \cite{friedman2001elements}. Furthermore, in high-dimensional settings the classic distributional results upon which p-values are constructed (the summary value for assessing the statistical relevance of a feature) are no longer valid, so the statistical guarantees for any selection based on those methods may completely break down \cite{sur2018modern}. 

The knockoff procedure is a two-step procedure that has recently emerged which avoids the pitfalls of the model-based global feature selection procedures by shifting the bulk of the statistical hypothesis from modeling $Y|\Xb$ to modeling the distribution $P^{\Xb}$ of $\Xb$ \cite{KN1, KN2}. False Discovery Rate control (i.e. control of the number of false positives in expectation) is guaranteed during the first step of the procedure, by assuming only knowledge of $P^{\Xb}$: during this step the analyst does not make any assumption on the behavior of $Y|\Xb$, and still FDR is controlled no matter which choices are made downstream. This first step constructs what are called knockoff variables $\tilde{\Xb}$ that are used as controls when running the second step of the procedure. That second step aims at getting a high number of true selections (referred to as the power of the procedure), which requires the analyst to fit some appropriate procedure to the joint data $(\Xb, Y)$ among a large pool of options. To summarize, the knockoff procedure decouples the assumptions underlying the statistical guarantee (knowledge of $P^{\Xb}$ to generate $\tilde{\Xb}$) from the assumptions needed to have a high power (fitting a model for $Y|\Xb$). 

We now introduce some notation for a more formal presentation of the knockoff procedure. For $\xb, \tilde{\xb} \in \mathbb{R}^d$, we define the concatenated vector by $[\xb, \tilde{\xb}] \in \mathbb{R}^{2d}$, and the following swap operation: given a subset $S \subset [d]$ of indices, the vector $[\xb,\tilde{\xb}]_{swap(S)}\in \mathbb{R}^d \times \mathbb{R}^d$ corresponds to $[\xb,\tilde{\xb}]$ where the coordinates indexed by $S$ are swapped between $\xb$ and $\tilde{\xb}$ (we use the same notation whenever the vectors are stacked in a matrix). Also, for any subset of indices $S\subset [d]$ we denote by $\xb_{S} \in \mathbb{R}^{|S|}$ the restriction of $\xb$ to coordinates in $S$. For a set $A$, denote $\#A$ its size.

Assuming we know the ground truth for the distribution $P^X$, the first step of the knockoff procedure is to obtain a \emph{knockoff} sample $\tilde{\Xb}$ that satisfies the following conditions \cite{KN2}: 

\begin{definition}[Knockoff sample]
A knockoff sample $\tilde{\Xb}$ of a $d$-dimensional random variable $\Xb$ is a $d$-dimensional random variable such that two properties are satisfied:
\vspace{-3mm}
\begin{itemize}
\setlength\itemsep{-0.2em}
\item Conditional independence: \hspace{3mm}$\tilde{\Xb} \independent Y | \Xb$ 
\item Exchangeability :  
\vspace{-2mm}
\[[\Xb,\tilde{\Xb}]_{swap(S)}\; \stackrel{d}{=}\; [\Xb,\tilde{\Xb}]\qquad  \forall S \subset \{1,\dots,d\}
\]
\end{itemize}
\vspace{-2mm}
where $\stackrel{d}{=}$ stands for equality in distribution.
\end{definition}

The first condition is immediately satisfied as long as knockoffs are sampled conditioned on the sample $\Xb$ without considering any information about $Y$. Satisfying the exchangeability condition is more technical and has recently received widespread attention \cite{KN3, jordon2018knockoffgan, gimenez2018knockoffs, liu2018auto, romano2018deep}. Indeed satisfying the exchangeability condition is the cornerstone of the knockoff procedure upon which the FDR control relies, thus the importance of having accurate mechanisms to generate valid knockoffs.

The next step requires the analyst to apply some procedure on the dataset $[\mathbb{X}, \tilde{\mathbb{X}}],\mathbb{Y}$ such that it constructs \emph{importance scores} $[\Tb, \tilde{\Tb}] \in \mathbb{R}_+^d \times \mathbb{R}_+^d$: a large value for $T_j$ (or $\tilde{T}_j$) indicates that the procedure considers that the feature $X_j$ (or $\tilde{X}_j$) has a strong influence on $Y$. Notice that here we consider the full dataset of i.i.d. samples $\mathbb{X}$, and the corresponding dataset of knockoff variables $\tilde{\mathbb{X}}$, that are generated sample-wise from $\mathbb{X}$. The importance scores are random variables constructed based on the whole set of samples. Some examples consist in taking the absolute values of the coefficients in a linear or logistic regression, the time of first entry of a feature in a LASSO path in a $L^1$-regularized regression, the drops in accuracy in a trained classifier by perturbing one feature at a time \cite{KN1, KN2, gimenez2018knockoffs, lu2018deeppink}. The only crucial requirement on such importance scores to be valid (for the knockoff procedure) is that they are \emph{associated} to their corresponding feature, and agnostic to the ordering of the remaining features. That is, for any $S\subset [d]$, if the analyst feeds the procedure with the dataset $[\mathbb{X}, \tilde{\mathbb{X}}]_{swap(S)},\mathbb{Y}$, then the resulting importance scores are $[\Tb, \tilde{\Tb}]_{swap(S)}$. 

The strength and flexibility of the knockoff procedure is that there is no need to assume any model on the joint distribution of $(\Xb, Y)$, and that we are free to choose any method that generates valid importance scores: the FDR control will hold because the knockoff importance scores $\tilde{T}_j$ will serve as controls for the $T_j$. For a target FDR level $q\in (0,1)$, the output of the procedure is the set $\hat{\mathcal{S}} = \{j : (T_j - \tilde{T}_j) \geq \hat{\tau}\}$ based on a data-driven threshold
\begin{equation}\label{tau-equation}
\hat{\tau} = \min \Big\{ t > 0 : \frac{1+ \#\{j: (T_j - \tilde{T}_j) \leq -t\}}{\#\{j: (T_j - \tilde{T}_j)\geq t\}} \leq q\Big\}
\end{equation}
The $j$th feature is selected when the difference between the importance scores of $X_j$ and $\tilde{X}_j$ is larger than $\hat{\tau}$.

\paragraph{Why subsetting the data first does not work} The first naive way to extend the knockoff procedure to a local setting is to run the whole procedure after restricting the initial dataset. For $\xb \in \mathbb{R}^d$, if we are interested in identifying the set of local non-nulls $\mathcal{H}_1^r(\xb)$, we could run the knockoff procedure starting with the set \mbox{$\{(\Xb_i, Y_i), \, s.t. \Vert \Xb_i -\xb \Vert_{\infty} \leq r\}$}. The reason why this approach is difficult to implement is that generating $\tilde{\Xb}$ for the conditional distribution of $\Xb$ being in a neighborhood of $\xb$ is generally not tractable. Restricting the amount of data harms how well we can fit our model of $P^{\Xb}$ to the data. Moreover conditioning on the values of $\Xb$ requires a different process for sampling the knockoffs $\tilde{\Xb}|\Xb$ than what is currently available \cite{KN1, KN2, gimenez2018knockoffs, lu2018deeppink}. In order to construct local selection sets at a given point $\xb$, we need to keep the original knockoff generation process intact, but construct local importance scores. Keeping the statistical guarantee in place requires reformulating the whole knockoff procedure in a local approach.

\subsection{Knockoff with Local Importance Scores}

Our extension of the knockoff procedure assumes that we have valid knockoffs available for the whole dataset: that is, the first step of constructing knockoffs is exactly the same as in the global knockoff procedure. By using the whole available dataset, by making the same assumptions as in the global knockoff setting, we can reuse all the available tools already developed for building valid knockoffs, and will not leave out any useful information which could hurt the validity of the statistical guarantees.

\begin{algorithm}[tb]
   \caption{Knockoffs with Local Importance Scores Feature Selection Procedure}
   \label{alg:LocalKnockoff}
\begin{algorithmic}
   \STATE {\bfseries Input:} Dataset $\mathbb{X}, \mathbb{Y}$, set of $L$ points $\mathcal{L}=\{\zb_l\}_{1\leq l\leq L}$, radius $r>0$.
   \STATE Generate dataset of knockoff variables $\tilde{\mathbb{X}}$.
   \FOR{$l=1$ {\bfseries to} $L$}
   \STATE Subsample data $(\mathbb{X}^{(l)}, \tilde{\mathbb{X}}^{(l)}, \mathbb{Y}^{(l)}) = \{(\Xb_{\cdot}, \tilde{\Xb}_{\cdot}, Y_{\cdot})$ such that \mbox{$\Vert \Xb_{\cdot}- \zb_l\Vert_{\infty} \leq r \; \& \;  \Vert \tilde{\Xb}_{\cdot} -\zb_l\Vert_{\infty} \leq r \}$}
   \STATE Generate ($r$-local) importance scores $\Tb^{(l)}, \tilde{\Tb}^{(l)}$ from $(\mathbb{X}^{(l)}, \tilde{\mathbb{X}}^{(l)}, \mathbb{Y}^{(l)})$
   \STATE Construct $\hat{\tau}^{(l)}$ as in equation~\ref{tau-equation} from $\Tb^{(l)}, \tilde{\Tb}^{(l)}$
   \STATE Construct $\hat{\mathcal{S}}_{l} = \{j \in [d]: \Tb^{(l)}_j - \tilde{\Tb}^{(l)}_j \geq \hat{\tau}^{(l)}\}$
   \ENDFOR
   \STATE {\bfseries Return:} Selected sets $(\Hat{\mathcal{S}}_{l})_{1\leq l \leq L}$
\end{algorithmic}
\end{algorithm}

Fix $r>0$, consider a set of $L$ points $\zb_1, \dots, \zb_L$ whose pairwise distances are lower bounded by $2r$. We now present Algorithm~\ref{alg:LocalKnockoff}, a procedure that constructs $L$ sets $\hat{\mathcal{S}}_1, \dots , \hat{\mathcal{S}}_L \subset [d]$ of selected features at each point $\zb_l$, such that we control $\text{FDR}^r$ at any desired FDR target level $q\in (0,1)$, uniformly for all sets. The crucial aspect of this procedure is that we are able to make a local statement for local FDR uniformly for several points. This algorithm starts by constructing knockoffs $\tilde{\mathbb{X}}$ from $\mathbb{X}$ as in the global knockoff procedure. It then subsets the whole dataset into $L$ datasets $(\mathbb{X}^{(l)}, \tilde{\mathbb{X}}^{(l)}, \mathbb{Y}^{(l)})$ where for each $1\leq l \leq L$ the samples in the $l$th subset are such that the original vector $\Xb_{\cdot}^{(l)}$ and the knockoff sample $\tilde{\Xb}_{\cdot}^{(l)}$ are in a $r$-ball neighborhood of $\zb_l$ (we denote by $\Xb_{\cdot}$ a generic sample of the vector $\Xb$ among the samples in $\mathbb{X}^{(l)}$). We now run construct $r$-local importance scores independently for each subset, that we use to construct a selected set of features as in the global knockoff procedure. As for the global knockoff procedure, a wise choice of importance scores may lead to a higher power of the procedure, i.e. a correct modelling of the relationship between the original features and the response for the dataset $(\mathbb{X}^{(l)}, \tilde{\mathbb{X}}^{(l)}, \mathbb{Y}^{(l)})$. As long as the importance scores satisfy the association requirement stated above, one can use different methods for each $l$: we can use the absolute coefficients when fitting a logistic regression \cite{KN2} for one of the subsets if we know it is a good model for that subset, and a more generic accuracy drop \cite{gimenez2018knockoffs} for another region where we may lack such knowledge.

\begin{theorem}\label{thm:guarantee}
The output of Algorithm~\ref{alg:LocalKnockoff} is such that $\text{FDR}^r$ is controlled uniformly for $\zb_l$. That is, denoting for each $1\leq l \leq L$,
\[
\text{FDP}^r(\zb_l) = \frac{|\hat{\mathcal{S}}_l \cap \mathcal{H}_0^r(\zb_l)|}{1\vee |\hat{\mathcal{S}}_l|}
\]
we have that:
\[
\sup_l \mathbb{E}[\text{FDP}^r(\zb_l)] \leq q
\]
\end{theorem}

In order to prove this result we need to show that the usual knockoffs can be reused in a local way: this requires to redefine and extend the distributional properties underlying the proof of FDR control from the global knockoff procedure as in \cite{KN2}. We present the technical tools needed to prove this theorem in Section~\ref{section:technical}.

\subsection{Algorithms for Partitioning the Feature Space into Meaningful Sub-regions}

We now present the last element needed to obtain a reasonable local selection procedure. The implementation of Algorithm~\ref{alg:LocalKnockoff} requires the analyst to indicate the points at which the selection sets need to be computed. Such knowledge may not be available: it even requires to identify some of the non-nulls that are needed to partition the space into sub-regions with different sets of non-nulls. We now present an approach to hierarchically identify relevant "switch" features for the purpose of partitioning the feature space and the values that these switch features take at the boundaries of the partitions, and a second one in Appendix~\ref{appendix:partition}. These algorithms based on heuristics only provide input values to Algorithm~\ref{alg:LocalKnockoff}. As such, the statistical guarantees described in Theorem~\ref{thm:guarantee} always hold.

\paragraph{Importance Score-based Partitioning}

Finding the best "switch" feature and its boundary value can be done in a greedy manner. The output will be relevant depending on the choice of importance scores that we use as a subroutine of this algorithm. Fix a subset $I \subset [d]$ of features, fix a set $\Lambda = (\lambda_1 , \dots, \lambda_L) \in \mathbb{R}^L$ of $L$ boundary values. For each pair $(c,\lambda)$ of feature and split value, we split the full matrix dataset $(\mathbb{X}, \mathbb{Y})$ into two sets $(\mathbb{X}^{>,c\lambda}, \mathbb{Y}^{>,c\lambda})$ and $(\mathbb{X}^{<,c\lambda}, \mathbb{Y}^{<,c\lambda})$ based on the value of the $d$th covariate of each sample $\Xb_i$ and compute importance scores $\Tb^{>,c\lambda}$ and $\Tb^{<,c\lambda}$ independently on each (notice that importance scores do not require the presence of knockoffs: moreover, importance scores are constructed based on a set of features agnostic to whether those features were original or knockoff). As a method to generate importance scores, we will train a classifier on the dataset and compute, for each feature, the accuracy drops when reevaluating the trained classifier on the same initial data except for one feature that we shuffle across samples. We then select the optimal split as the one with highest gap in importance scores in $L^{\infty}$ norm between the two splits. We can then recursively, depending on the amount of data, run the algorithm on each of the splits, or, if it corresponds to the final step, return a set of vectors lying in distinct subregions, spaced by at least $2r$.

\begin{algorithm}[tb]
   \caption{One-Step Greedy Feature Space Partition}
   \label{alg:greedypartition}
\begin{algorithmic}
   \STATE {\bfseries Input:} Subset of indices I, Split values $\Lambda$, Dataset $\{\mathbb{X}, \mathbb{Y}\}$, radius r.
   \FOR{$d$ {\bfseries in} I}
   \FOR{$\lambda$ {\bfseries in} $\Lambda$}
   \STATE Partition the data $(\mathbb{X}, \mathbb{Y})$ according to $X_{\cdot d} > \lambda$ or $X_{\cdot d} < \lambda$ into $(\mathbb{X}^{>,c\lambda}, \mathbb{Y}^{>,c\lambda})$ and $(\mathbb{X}^{<,c\lambda}, \mathbb{Y}^{<,c\lambda})$.
   \STATE Generate importance scores $\Tb^{>,c\lambda}$ and $\Tb^{<,c\lambda}$.
   \STATE Compute $gap_{d, \lambda} = \Vert\Tb^{>,c\lambda} - \Tb^{<,c\lambda} \Vert_{\infty}$
   \ENDFOR
   \ENDFOR
   \STATE Select $d^*, \lambda^* = \argmax gap_{d,\lambda}$.
   \IF{Final recursive step}
   \STATE Return $\zb_{<}, \zb_{>}$ s.t. $\zb_{<,d^*}\!  \leq\! \lambda^*\! -\! r$ and \mbox{$\zb_{>,d^*} \!\geq\! \lambda^* \!+\! r$}.
   \ELSE
   \STATE Return $(\mathbb{X}^{>,c\lambda}, \mathbb{Y}^{>,c\lambda})$ and \mbox{$(\mathbb{X}^{<,c\lambda}, \mathbb{Y}^{<,c\lambda})$}.
   \ENDIF
\end{algorithmic}
\end{algorithm}

The main disadvantage of this method is the computational cost. Indeed, for a given subset of relevant features, identifying the optimal choice for the switch feature and the boundary value is quadratic in the size of the subset. As such we may want to choose an appropriate subset $I$ and set $\Lambda$ as inputs for Algorithm~\ref{alg:greedypartition} (any global feature selection procedure can be used to identify an appropriate set $I$). Given that we need to repeatedly fit the model, this method may become prohibitive whenever the number of relevant features is too high (i.e. non sparse settings or high-dimensional settings). We present an alternative approach to overcome computational bottlenecks in Appendix~\ref{appendix:partition}.

\section{Experiments}

We now run experiments and the first goal is to show that our main theorem holds. Among the many available distributions for $P^{\Xb}$ that allow for efficient sampling of knockoffs, we choose the Hidden Markov Model (HMM) \cite{KN3} to illustrate a GWAS study. We therefore sample the datasets $\mathbb{X},\tilde{\mathbb{X}}$ so that $\mathbb{X}$ simulates the SNPs of a cohort of patients, i.e. a matrix of 0, 1, and 2. We generate a complex response $Y$ in the following way: we randomly pick a subset of $q$ features $\mathcal{S}_0 = \{s_1,\dots,s_q\}$ that correspond to the global non-nulls. However, we design the response in such way that not all global non-nulls are local non-nulls at every point. Our design of the response begins by choosing switch features: based on the value of $X_{s_1}$ (greater or lower than 1.5) we select the first or second half of the remaining global non-null indices $\{s_2,\dots,s_q\}$, and we repeat one more time this process (i.e. for example whenever $X_{s_1}>1.5$ look at whether $X_{s_2}>1.5$, in which case we end up picking the first half of $\{s_3,\dots, s_{q/2}\}$). We end up, for each individual sample $\Xb_i = (X_{i,1},\dots,X_{i,d})$ in the dataset $\mathbb{X}$, with a subset of indices $\mathcal{S}_i\subset \mathcal{S}_0 \subset [d]$ (in the case where $X_{i,s_1}, X_{i,s_2} >1.5$, we have then $\mathcal{S}_i = \{s_3,\dots,s_{q/4}\}$). We then generate $Y_i$ as a Bernoulli whose probability is given by a linear combination of the values $\Xb_{i,\mathcal{S}_i}$ (i.e. can be modeled through a logistic regression on that subset of the space). This way, the set of local non-nulls varies depending on the values of the switch features.

\begin{figure}[ht]
\centering
\includegraphics[width=\linewidth]{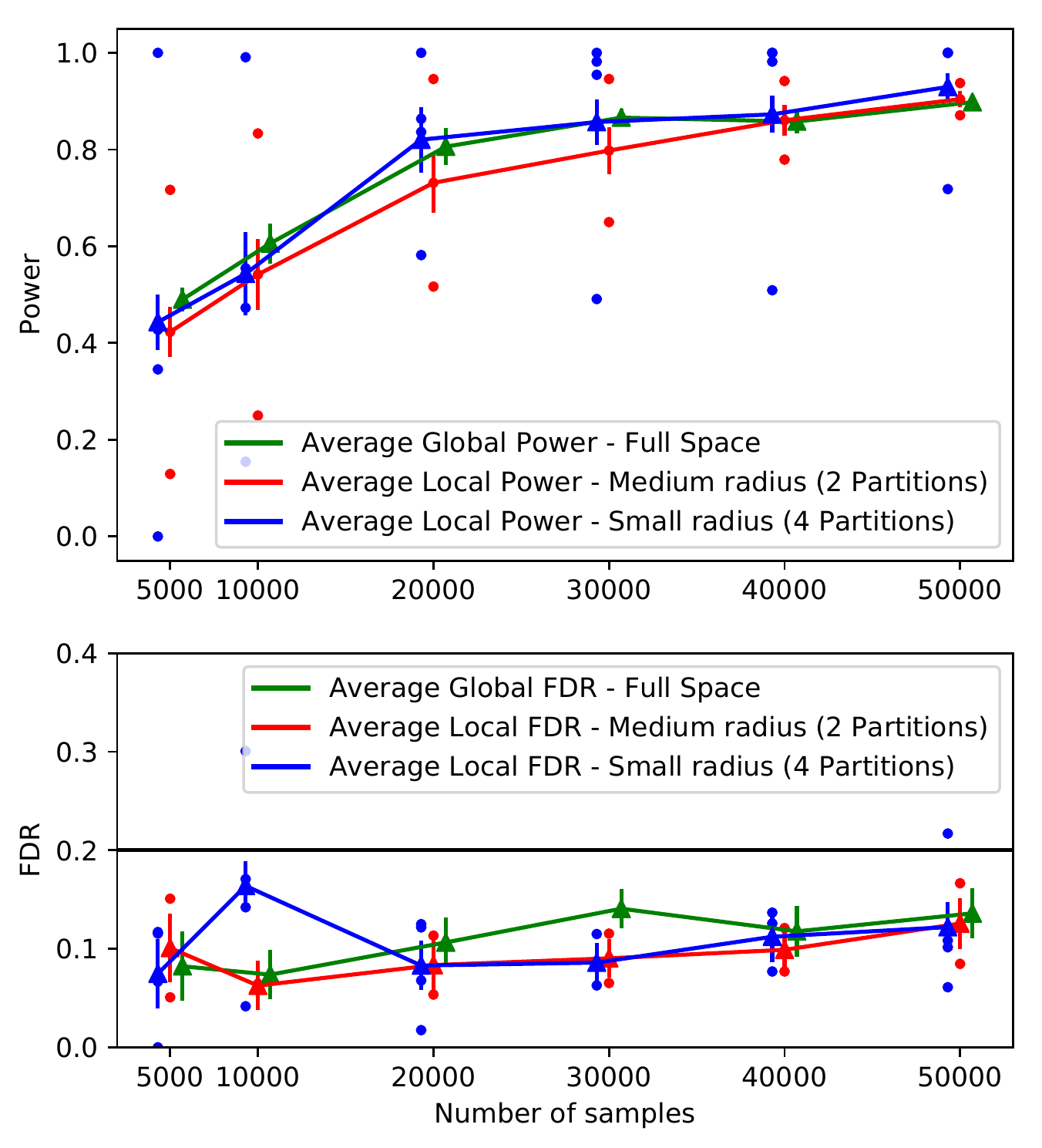}
\vspace{-0.2cm}
\caption{\textbf{Local Power and Local FDR as a Function of Sample size.} For several partitions, the dots indicate the (averaged over 20 runs) local power (top) / local FDR (bottom) at each of the partitions (1, 2 or 4). The lines correspond to the averages across partitions and 20 runs, with estimates of the standard errors. We report the target FDR at 0.2 with a horizontal line.}
\label{fig:1}
\end{figure}

We now run Algorithm~\ref{alg:LocalKnockoff} at three "resolution levels", i.e. for $L =1, 2$ and $4$ points $(\zb_l)_{1\leq l \leq L}$ (so that we respectively construct 1,2,4 selection sets $(\Hat{\mathcal{S}}_l)_{1\leq l \leq L}$), using each time an optimal choice of points that we assume are given and correspond to the optimal partitions in our response generating process. Our target FDR is $q=0.2$ and locally we consider as importance scores the absolute values of the coefficients in a logistic regression. We report the results in Figure~\ref{fig:1}, where the dots refer (for each resolution level) to the local FDR/power at each individual set of the partition. It is crucial to notice that here the FDR/Power scale is related to the procedure: as we increase the partition we consider local FDR/Power on a smaller scale: at the highest resolution, i.e. $L=4$, for the point $\zb_l$ corresponding to the region $X_{s_1}, X_{s_2} >1.5$, we will evaluate the FDR and power of the output $\Hat{\mathcal{S}}_l$ of the procedure locally with respect to $\mathcal{S}_{0l} = \{s_3,\dots,s_{q/4}\}$. Our procedure is able, in every setting, \emph{to retrieve the correct subset of non-nulls at the corresponding scale while controlling for FDR at the corresponding scale}. That is, the run with just 1 point corresponds to the usual global feature selection problem and selects features that are global nulls. Whenever we choose 2 or 4 points, we output several selection sets that correspond to the local nulls on the corresponding regions.

However to appreciate the real interest of the local feature selection problem one needs to consider what happens to the features that are non-null in a subspace but null in another (for example $X_{s_3}$ is non-null only whenever $X_{s_1}, X_{s_2} > 1.5$). The fact that our procedure controls local FDR indicates that, at a given subspace, it is properly identifying the true local non-nulls, and \emph{not selecting features that are null in that subspace, albeit non-null somewhere else}. Figure~\ref{fig:2} illustrates this phenomenon: we run again the procedure (for different values of $L$) but report the averaged local FDR across the four partitions whenever $L=4$. This indicates that, if we were only interested in a phenomenon in a subregion of the feature space, the output of the global feature selection procedure would report a large proportion of features that are actually irrelevant in the corresponding subregion.

\begin{figure}[ht]
\centering
\includegraphics[width=\linewidth]{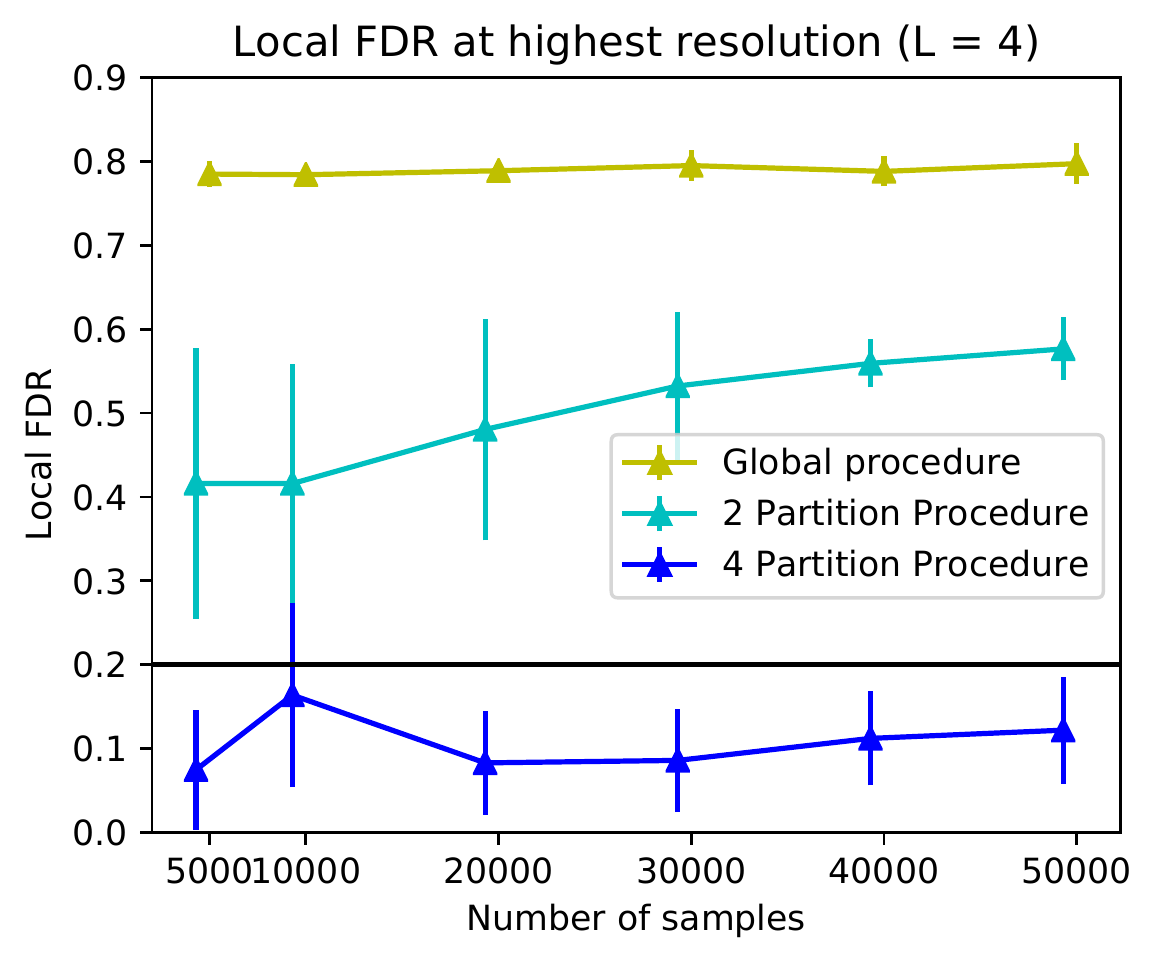}
\vspace{-0.2cm}
\caption{\textbf{Local FDR at highest resolution for procedures at different resolutions.} Whenever we evaluate the results of a low-resolution procedure (global feature selection) with a high resolution FDR (small radius with 4 partitions), the coarser methods can not control the finer FDR.}
\label{fig:2}
\end{figure}

Notice also that our simulations were done in dimension 200, and we vary the size of the dataset to show that our method still controls local FDR even though the number of points in a given subregion is very limited: we indeed have lower power in these settings, but our statistical guarantees always hold.

\section{Theoretical Analysis of Local Feature FDR}\label{section:technical}

Our main result presented in Theorem~\ref{thm:guarantee} holds regardless of the method used to partition the space. It is based upon a generalization of several notions defined in the first introductions of the knockoff procedure \cite{KN1,KN2} which we now present.

\paragraph{Local swaps and Local Exchangeability}  

We extend the definition of the swap to mappings from $ \mathbb{R}^d \times \mathbb{R}^d$ to $ \mathbb{R}^d \times \mathbb{R}^d$, where the set of indices $S$ used to swap depends on the inputs of the mapping. 

\begin{definition}\label{def-knockoff-swap}
We call a mapping $\sigma : \mathbb{R}^d \rightarrow \mathcal{P}([d])$ a generic swap or a swap. In addition, we say that a swap is a local swap if for any $\xb,\zb \in \mathbb{R}^d$, 
\[\xb_{[d]\setminus \sigma(\xb)} = \zb_{[d]\setminus \sigma(\xb)} \quad \Rightarrow \quad \sigma(\xb) = \sigma(\zb)\]

Given a mapping
\begin{equation*}
    (\Fb,\tilde{\Fb}) : \begin{cases}
    \mathbb{R}^d \times \mathbb{R}^d \rightarrow     \mathbb{R}^d \times \mathbb{R}^d
    \\ (\xb, \tilde{\xb}) \mapsto 
    \begin{aligned}
    \big(F_1(\xb, &\tilde{\xb}), \dots, F_d(\xb, \tilde{\xb}), 
    \\
    \vspace{-1mm}
    &\tilde{F}_1(\xb, \tilde{\xb}), \dots, \tilde{F}_d(\xb, \tilde{\xb})\big)
    \end{aligned}
    \end{cases}
\end{equation*}
and swap $\sigma$, define the operation $[\Fb,\tilde{\Fb}]_{swap(\sigma)}$ as the mapping
$
    [\Fb,\tilde{\Fb}]_{swap(\sigma)}: (\xb, \tilde{\xb}) \mapsto \big[(\Fb,\tilde{\Fb})(\xb, \tilde{\xb})\big]_{swap(\sigma(\xb))}
$.
\end{definition}

We develop further this notion and additional notation in Appendix~\ref{appendix:theory}. Our goal now is to show that the exchangeability condition that defines a knockoff variable implies a stronger distributional result. 

\begin{proposition}\label{extended-exchangeability}
Let $\sigma$ be a local swap. If $\tilde{\Xb}$ is a knockoff random variable for $\Xb$ (i.e. satisfies exchangeability), then 
\begin{equation}\label{general-exchangeability}
[\Xb,\tilde{\Xb}]_{swap(\sigma)} \stackrel{d}{=} [\Xb,\tilde{\Xb}]
\end{equation}
which we refer to as local exchangeability. If $\sigma \subset \mathcal{H}_0^0$, then 
\begin{equation}
[\Xb,\tilde{\Xb}]_{swap(\sigma)}, Y \stackrel{d}{=} [\Xb,\tilde{\Xb}], Y
\end{equation}
\end{proposition}

We prove this result in Section~\ref{proof-extended-exchangeability}. This result extends the exchangeability property and the Lemma 3.2 in \cite{KN2}. Instead of swapping a fixed set of features, we now allow the swapping indices to depend on the features. Notice that the knockoffs $\tilde{\Xb}$ are constructed as in the general case, the local exchangeability does not require a different definition for knockoff variables. We extend the definition of a swap to probability distributions: for $\mu \in Pr(\mathbb{R}^d \times \mathbb{R}^d)$, we denote $\mu_{swap(\sigma)} := \mathcal{L}([\Xb,\tilde{\Xb}]_{swap(\sigma)})$ whenever $\mu = \mathcal{L}([\Xb,\tilde{\Xb}])$. Abusing notation, whenever $\mu = \mathcal{L}([\Xb,\tilde{\Xb}], Y)$ we will still denote $\mu_{swap(\sigma)} := \mathcal{L}([\Xb,\tilde{\Xb}]_{swap(\sigma)}, Y)$.

\paragraph{Local Feature Statistics}
The next step is to extend the construction of feature statistics to the local setting.
\begin{definition}
Define local importance scores as a mapping:
\begin{equation} \Phi : 
    \begin{cases}
    Pr(\mathbb{R}^d \times \mathbb{R}^d\times \mathbb{R}) \longrightarrow \big(\mathbb{R}^d \rightarrow \mathbb{R}^d\times \mathbb{R}^d\big) 
    \\ \mu \mapsto (\Tb_{\mu}, \tilde{\Tb}_{\mu})
    \end{cases}
\end{equation}
where 
\begin{equation}
    (\Tb_{\mu}, \tilde{\Tb}_{\mu}) : \begin{cases}
    \mathbb{R}^d \rightarrow \mathbb{R}^d\times \mathbb{R}^d 
    \\ \zb \mapsto (\Tb_{\mu}(\zb), \tilde{\Tb}_{\mu}(\zb))
    \end{cases}
\end{equation}
such that, for any $S\subset [d]$, we have 
\[
\Phi(\mu_{swap(S)}) = [\Phi(\mu)]_{swap(S)}
\]

For $r>0$, we say that such importance scores $\Phi$ are $r$-local if, for any $\mu \in  Pr(\mathbb{R}^d \times \mathbb{R}^d\times \mathbb{R})$, we have that $\Phi(\mu)(\zb) = (\Tb_{\mu}(\zb), \tilde{\Tb}_{\mu}(\zb))$ only depends on $\mu$ through the restriction of $\mu$ to $B(\zb, r) \times B(\zb, r) \times \mathbb{R}$. That is, if $\mu, \mu'$ are two probability measures on $\mathbb{R}^d \times \mathbb{R}^d\times \mathbb{R}$ such that they coincide on $B(\zb, r) \times B(\zb, r) \times \mathbb{R}$, then  $(\Tb_{\mu}(\zb), \tilde{\Tb}_{\mu}(\zb)) =  (\Tb_{\mu'}(\zb), \tilde{\Tb}_{\mu'}(\zb))$.
\end{definition}

The next goal consists in translating the swap operation in $\mu_{swap(\sigma)} = \mathcal{L}([\Xb,\tilde{\Xb}]_{swap(\sigma)}, Y)$ into a swap of $[\Tb_{\mu}, \tilde{\Tb}_{\mu}]_{swap(\sigma)}$. This step does not require $\tilde{\Xb}$ to be a knockoff of $\Xb$: in what follows we do not make any assumption on $\mu$. Notice that the swap operation has been defined (Definition ~\ref{def-knockoff-swap}) as a transformation of a mapping $\mathbb{R}^d\times \mathbb{R}^d\rightarrow \mathbb{R}^d\times \mathbb{R}^d $, but it can be immediately extended to mappings $\mathbb{R}^d \rightarrow \mathbb{R}^d\times \mathbb{R}^d$. We are able to relate $[\Tb_{\mu_{swap(\sigma)}}, \tilde{\Tb}_{\mu_{swap(\sigma)}}]$ and $[\Tb_{\mu}, \tilde{\Tb}_{\mu}]_{swap(\sigma)}$ if we assume locality constraints on the importance scores.

\begin{definition}
For $\sigma : \mathbb{R}^d \rightarrow \mathcal{P}([d])$ local swap, and $r>0$, define
\begin{align*}
A^r & = \{\zb \in \mathbb{R}^d, \; \sigma(\zb) = \sigma(\yb) \; \forall \yb \in B(\zb,r)\}
\\ \sigma^r & : \zb \mapsto \begin{cases}
\sigma(\zb) \quad \text{if} \;\; \zb \in A^r
\\ \emptyset \quad \text{else}
\end{cases}
\end{align*}
\end{definition}

\begin{proposition}\label{swap-scores}
Assume there exists $r>0$ such that the importance scores are $r$-local. Then, assuming $A^{r}$ is non-empty, for $\zb \in A^{r}$, we have:
\[
[(\Tb_{\mu}(\zb), \tilde{\Tb}_{\mu}(\zb))]_{swap(\sigma(\!\zb\!))}\! =\! [\Tb_{\mu_{swap(\sigma)}}\!(\zb), \tilde{\Tb}_{\mu_{swap(\sigma)}}\!(\zb)]
\]
\end{proposition}

We prove this result in Section~\ref{proof:swap-scores}. Whenever $\mu = \mathcal{L}([\Xb, \tilde{\Xb}], Y)$ where $\tilde{\Xb}$ is a knockoff of $\Xb$, consider a local swap $\sigma$ such that  $\sigma \subset \mathcal{H}_0^0$. By proposition~\ref{extended-exchangeability}, we get that $\mu = \mu_{swap(\sigma)}$, and therefore $(\Tb_{\mu}, \tilde{\Tb}_{\mu}) = (\Tb_{\mu_{swap(\sigma)}}, \tilde{\Tb}_{\mu_{swap(\sigma)}})$. In practice, we can imagine that instead of using $\mu$ as an input when constructing importance scores, we use $\hat{\mu}_n$, an empirical measure defined by the dataset of $n$ i.i.d. samples from $\mu$ that we feed as input to the algorithm. Therefore a consequence of proposition~\ref{extended-exchangeability} will be that, taking also into account the eventual randomness when generating importance scores, we have for any $\zb \in \mathbb{R}^d$,
\[
(\Tb_{\hat{\mu}_{n}}(\zb), \tilde{\Tb}_{\hat{\mu}_n}(\zb)) \stackrel{d}{=} (\Tb_{\hat{\mu}_{n,swap(\sigma)}}(\zb), \tilde{\Tb}_{\hat{\mu}_{n,swap(\sigma)}}(\zb))
\]

We now combine this equality with that of Proposition~\ref{swap-scores}. Fix $r>0$ such that we have $r$-local importance scores. Consider a set of $L$ points $\zb_1, \dots \zb_N \in \mathbb{R}^d$ that are pairwise $2r$ far apart, that is, for any $1\!\leq \!l,l'\!\leq\! N$, \mbox{$\Vert \zb_l - \zb_{l'} \Vert_{\infty} \geq 2r$}. We can now conclude thanks to Proposition~\ref{sigma-r} in the Appendix, which implies the flip-sign property introduced first in \cite{KN2}. This is all that is needed to construct an adaptive threshold $\tau^{l}$ based on $\Tb_{\hat{\mu}_{n}}(\zb_l), \tilde{\Tb}_{\hat{\mu}_n}(\zb_l)$, and the corresponding selected features $\hat{\mathcal{S}}_l$ such that FDR is controlled. 

That is, according to Theorem 3.4 in \cite{KN2}, the set $\Hat{\mathcal{S}}_l$ is such that  
\[
\mathbb{E}[\frac{|\hat{\mathcal{S}}_l\cap \mathcal{H}_0^r(\xb)|}{1\vee | \hat{\mathcal{S}}_l|}] \leq q
\]
hence the result. We provide additional details in Appendix~\ref{appendix:theory}.

\section{Discussion}

We develop a new framework for selecting features conditionally on being in a specific region of the feature space, for which we define the notions of local null features and local FDR. The fact that FDR is controlled even though the analyst does not need to model the relationship between the outcome and the feature variables makes the knockoff procedure stand out among feature selection procedures. We extend the usage of the knockoff procedure to deal with our local feature selection problem, in such way that there is no need to strengthen the assumptions required for the statistical guarantees to hold. 

\subsubsection*{Acknowledgments}

J.R.G. is supported by a Stanford Graduate Fellowship. J.Z. is supported by a Chan–Zuckerberg Biohub Investigator grant and National Science Foundation (NSF) Grant CRII 1657155.

\bibliography{main}

\begin{thebibliography}{25}
\providecommand{\natexlab}[1]{#1}
\providecommand{\url}[1]{\texttt{#1}}
\expandafter\ifx\csname urlstyle\endcsname\relax
  \providecommand{\doi}[1]{doi: #1}\else
  \providecommand{\doi}{doi: \begingroup \urlstyle{rm}\Url}\fi

\bibitem[Barber et~al.(2015)Barber, Cand{\`e}s, et~al.]{KN1}
Barber, R.~F., Cand{\`e}s, E.~J., et~al.
\newblock Controlling the false discovery rate via knockoffs.
\newblock \emph{The Annals of Statistics}, 43\penalty0 (5):\penalty0
  2055--2085, 2015.

\bibitem[Benjamini \& Hochberg(1995)Benjamini and
  Hochberg]{benjamini1995controlling}
Benjamini, Y. and Hochberg, Y.
\newblock Controlling the false discovery rate: a practical and powerful
  approach to multiple testing.
\newblock \emph{Journal of the royal statistical society. Series B
  (Methodological)}, pp.\  289--300, 1995.

\bibitem[Benjamini \& Yekutieli(2001)Benjamini and
  Yekutieli]{benjamini2001control}
Benjamini, Y. and Yekutieli, D.
\newblock The control of the false discovery rate in multiple testing under
  dependency.
\newblock \emph{Annals of statistics}, pp.\  1165--1188, 2001.

\bibitem[Cand{\`e}s et~al.(2018)Cand{\`e}s, Fan, Janson, and Lv]{KN2}
Cand{\`e}s, E., Fan, Y., Janson, L., and Lv, J.
\newblock Panning for gold:‘model-x’knockoffs for high dimensional
  controlled variable selection.
\newblock \emph{Journal of the Royal Statistical Society: Series B (Statistical
  Methodology)}, 2018.

\bibitem[Consortium et~al.(2015)]{1000genomes}
Consortium, . G.~P. et~al.
\newblock A global reference for human genetic variation.
\newblock \emph{Nature}, 526\penalty0 (7571):\penalty0 68, 2015.

\bibitem[Efron(2012)]{efron2012large}
Efron, B.
\newblock \emph{Large-scale inference: empirical Bayes methods for estimation,
  testing, and prediction}, volume~1.
\newblock Cambridge University Press, 2012.

\bibitem[Friedman et~al.(2001)Friedman, Hastie, and
  Tibshirani]{friedman2001elements}
Friedman, J., Hastie, T., and Tibshirani, R.
\newblock \emph{The elements of statistical learning}, volume~1.
\newblock Springer series in statistics New York, NY, USA:, 2001.

\bibitem[Gimenez et~al.(2018)Gimenez, Ghorbani, and Zou]{gimenez2018knockoffs}
Gimenez, J.~R., Ghorbani, A., and Zou, J.
\newblock Knockoffs for the mass: new feature importance statistics with false
  discovery guarantees.
\newblock \emph{arXiv preprint arXiv:1807.06214}, 2018.

\bibitem[Hirschhorn \& Daly(2005)Hirschhorn and Daly]{hirschhorn2005genome}
Hirschhorn, J.~N. and Daly, M.~J.
\newblock Genome-wide association studies for common diseases and complex
  traits.
\newblock \emph{Nature Reviews Genetics}, 6\penalty0 (2):\penalty0 95, 2005.

\bibitem[Holm(1979)]{holm1979simple}
Holm, S.
\newblock A simple sequentially rejective multiple test procedure.
\newblock \emph{Scandinavian journal of statistics}, pp.\  65--70, 1979.

\bibitem[Jordon et~al.(2019)Jordon, Yoon, and van~der
  Schaar]{jordon2018knockoffgan}
Jordon, J., Yoon, J., and van~der Schaar, M.
\newblock Knockoff{GAN}: Generating knockoffs for feature selection using
  generative adversarial networks.
\newblock In \emph{International Conference on Learning Representations}, 2019.
\newblock URL \url{https://openreview.net/forum?id=ByeZ5jC5YQ}.

\bibitem[Koh \& Liang(2017)Koh and Liang]{koh2017understanding}
Koh, P.~W. and Liang, P.
\newblock Understanding black-box predictions via influence functions.
\newblock \emph{arXiv preprint arXiv:1703.04730}, 2017.

\bibitem[Lippert et~al.(2011)Lippert, Listgarten, Liu, Kadie, Davidson, and
  Heckerman]{lippert2011fast}
Lippert, C., Listgarten, J., Liu, Y., Kadie, C.~M., Davidson, R.~I., and
  Heckerman, D.
\newblock Fast linear mixed models for genome-wide association studies.
\newblock \emph{Nature methods}, 8\penalty0 (10):\penalty0 833, 2011.

\bibitem[Lipton(2016)]{lipton2016mythos}
Lipton, Z.~C.
\newblock The mythos of model interpretability.
\newblock \emph{arXiv preprint arXiv:1606.03490}, 2016.

\bibitem[Liu \& Zheng(2018)Liu and Zheng]{liu2018auto}
Liu, Y. and Zheng, C.
\newblock Auto-encoding knockoff generator for fdr controlled variable
  selection.
\newblock \emph{arXiv preprint arXiv:1809.10765}, 2018.

\bibitem[Lu et~al.(2018)Lu, Fan, Lv, and Noble]{lu2018deeppink}
Lu, Y., Fan, Y., Lv, J., and Noble, W.~S.
\newblock Deeppink: reproducible feature selection in deep neural networks.
\newblock In \emph{Advances in Neural Information Processing Systems}, pp.\
  8690--8700, 2018.

\bibitem[McCarthy et~al.(2008)McCarthy, Abecasis, Cardon, Goldstein, Little,
  Ioannidis, and Hirschhorn]{mccarthy2008genome}
McCarthy, M.~I., Abecasis, G.~R., Cardon, L.~R., Goldstein, D.~B., Little, J.,
  Ioannidis, J.~P., and Hirschhorn, J.~N.
\newblock Genome-wide association studies for complex traits: consensus,
  uncertainty and challenges.
\newblock \emph{Nature reviews genetics}, 9\penalty0 (5):\penalty0 356, 2008.

\bibitem[Miller(2002)]{miller2002subset}
Miller, A.
\newblock \emph{Subset selection in regression}.
\newblock Chapman and Hall/CRC, 2002.

\bibitem[O'donovan \& Livingston(2010)O'donovan and Livingston]{o2010brca1}
O'donovan, P.~J. and Livingston, D.~M.
\newblock Brca1 and brca2: breast/ovarian cancer susceptibility gene products
  and participants in dna double-strand break repair.
\newblock \emph{Carcinogenesis}, 31\penalty0 (6):\penalty0 961--967, 2010.

\bibitem[Romano et~al.(2018)Romano, Sesia, and Cand{\`e}s]{romano2018deep}
Romano, Y., Sesia, M., and Cand{\`e}s, E.~J.
\newblock Deep knockoffs.
\newblock \emph{arXiv preprint arXiv:1811.06687}, 2018.

\bibitem[Sesia et~al.(2017)Sesia, Sabatti, and Cand{\`e}s]{KN3}
Sesia, M., Sabatti, C., and Cand{\`e}s, E.~J.
\newblock Gene hunting with knockoffs for hidden markov models.
\newblock \emph{arXiv preprint arXiv:1706.04677}, 2017.

\bibitem[Sun et~al.(2010)Sun, Todorovic, and Goodison]{sun2010local}
Sun, Y., Todorovic, S., and Goodison, S.
\newblock Local-learning-based feature selection for high-dimensional data
  analysis.
\newblock \emph{IEEE transactions on pattern analysis and machine
  intelligence}, 32\penalty0 (9):\penalty0 1610--1626, 2010.

\bibitem[Sur \& Cand{\`e}s(2018)Sur and Cand{\`e}s]{sur2018modern}
Sur, P. and Cand{\`e}s, E.~J.
\newblock A modern maximum-likelihood theory for high-dimensional logistic
  regression.
\newblock \emph{arXiv preprint arXiv:1803.06964}, 2018.

\bibitem[Tang et~al.(2014)Tang, Alelyani, and Liu]{tang2014feature}
Tang, J., Alelyani, S., and Liu, H.
\newblock Feature selection for classification: A review.
\newblock \emph{Data classification: Algorithms and applications}, pp.\ ~37,
  2014.

\bibitem[Tibshirani(1996)]{tibshirani1996regression}
Tibshirani, R.
\newblock Regression shrinkage and selection via the lasso.
\newblock \emph{Journal of the Royal Statistical Society. Series B
  (Methodological)}, pp.\  267--288, 1996.

\end{thebibliography}
\bibliographystyle{icml2019}

\newpage

\appendix

\section{Theoretical Analysis of Local Feature FDR}\label{appendix:theory}

\begin{definition}
We call a mapping $\sigma : \mathbb{R}^d \rightarrow \mathcal{P}([d])$ a generic swap or a swap. In addition, we say that a swap is a local swap if for any $\xb,\zb \in \mathbb{R}^d$, 
\[\xb_{[d]\setminus \sigma(\xb)} = \zb_{[d]\setminus \sigma(\xb)} \quad \Rightarrow \quad \sigma(\xb) = \sigma(\zb)\]

Given a mapping
\begin{equation*}
    (\Fb,\tilde{\Fb}) : \begin{cases}
    \mathbb{R}^d \times \mathbb{R}^d \rightarrow     \mathbb{R}^d \times \mathbb{R}^d
    \\ (\xb, \tilde{\xb}) \mapsto 
    \begin{aligned}
    \big(F_1(\xb, &\tilde{\xb}), \dots, F_d(\xb, \tilde{\xb}), 
    \\
    \vspace{-1mm}
    &\tilde{F}_1(\xb, \tilde{\xb}), \dots, \tilde{F}_d(\xb, \tilde{\xb})\big)
    \end{aligned}
    \end{cases}
\end{equation*}
and swap $\sigma$, define the operation $[\Fb,\tilde{\Fb}]_{swap(\sigma)}$ as the mapping
$
    [\Fb,\tilde{\Fb}]_{swap(\sigma)}: (\xb, \tilde{\xb}) \mapsto \big[(\Fb,\tilde{\Fb})(\xb, \tilde{\xb})\big]_{swap(\sigma(\xb))}
$.
\end{definition}
It is important to clearly identify the input space of the swap $\sigma$. The output of the swap operation on a mapping is again a mapping with the same input space, so for example we can iterate swap operations on a given mapping $(\Fb,\tilde{\Fb})$: for $\sigma_1,\sigma_2$  swaps the following is well defined:
$
    \big[[\Fb,\tilde{\Fb}]_{swap(\sigma_1)}\big]_{swap(\sigma_2)} .
$
Both $\sigma_1, \sigma_2$ are evaluated on a same given point $\xb$ of the space, and the final mapping corresponds to a point-wise concatenation of the swaps. As a consequence, the order of the swap operations does not matter. In particular, for $\sigma = \sigma_1 = \sigma_2$, we have
\[
\big[[\Fb,\tilde{\Fb}]_{swap(\sigma)}\big]_{swap(\sigma)} = (\Fb, \tilde{\Fb})
\]
Such mappings from $ \mathbb{R}^d \times \mathbb{R}^d$ to $\mathbb{R}^d \times \mathbb{R}^d$ can be concatenated: for $(\Fb, \tilde{\Fb})$, and $(\Gb, \tilde{\Gb})$ we denote $(\Fb, \tilde{\Fb}) \circ (\Gb, \tilde{\Gb})$ the concatenated mapping. Note that concatenation and swap operations do not behave well. For example, in the general case,
\begin{align*}
[(\Fb, \tilde{\Fb}) \circ (\Gb, \tilde{\Gb})]_{swap(\sigma)} &\neq  (\Fb, \tilde{\Fb}) \circ [\Gb, \tilde{\Gb}]_{swap(\sigma)}
\\ [(\Fb, \tilde{\Fb}) \circ (\Gb, \tilde{\Gb})]_{swap(\sigma)} &\neq  [\Fb, \tilde{\Fb}]_{swap(\sigma)} \circ (\Gb, \tilde{\Gb})
\end{align*}

One particular case is the identity mapping that we denote $(\Fb^{Id},\tilde{\Fb}^{Id})$, i.e. $\Fb_j^{Id}(\xb,\tilde{\xb}) = x_j$ and $\tilde{\Fb}^{Id}_j(\xb,\tilde{\xb}) = \tilde{x}_j$. Whenever we consider random variables $\Xb, \tilde{\Xb}$ we denote $ [\Xb, \tilde{\Xb}]_{swap(\sigma)} := [\Fb^{Id},\tilde{\Fb}^{Id}]_{swap(\sigma)}(\Xb, \tilde{\Xb})$ the random variable resulting from applying the swapped identity map. In addition, if $\sigma$ is constant equal to $S \subset [d]$, then we go back to the previous definition of swap $[\Xb,\tilde{\Xb}]_{swap(S)}$ in \cite{KN2}. Whenever we consider an identity mapping and a local swap $\sigma$, we have the immediate following result:
\[
( \Fb^{Id},\tilde{\Fb}^{Id}) = [\Fb^{Id},\tilde{\Fb}^{Id}]_{swap(\sigma)} \circ [\Fb^{Id},\tilde{\Fb}^{Id}]_{swap(\sigma)}
\]

Our goal now is to show that the exchangeability condition that defines a knockoff variable implies a stronger distributional result. 

\begin{proposition}
Let $\sigma$ be a local swap. If $\tilde{\Xb}$ is a knockoff random variable for $\Xb$ (i.e. satisfies exchangeability), then 
\begin{equation}
[\Xb,\tilde{\Xb}]_{swap(\sigma)} \stackrel{d}{=} [\Xb,\tilde{\Xb}]
\end{equation}
which we refer to as local exchangeability. If $\sigma \subset \mathcal{H}_0^0$, then 
\begin{equation}
[\Xb,\tilde{\Xb}]_{swap(\sigma)}, Y \stackrel{d}{=} [\Xb,\tilde{\Xb}], Y
\end{equation}
\end{proposition}

We prove this result in Section~\ref{proof-extended-exchangeability}. This result extends the exchangeability property and the Lemma 3.2 in \cite{KN2}. Instead of swapping a fixed set of features, we now allow the swapping indices to depend on the features. Notice that the knockoffs $\tilde{\Xb}$ are constructed as in the general case, the local exchangeability does not require a different definition for knockoff variables. We extend the definition of a swap to probability distributions: for $\mu \in Pr(\mathbb{R}^d \times \mathbb{R}^d)$, we denote $\mu_{swap(\sigma)} := \mathcal{L}([\Xb,\tilde{\Xb}]_{swap(\sigma)})$ whenever $\mu = \mathcal{L}([\Xb,\tilde{\Xb}])$. Abusing notation, whenever $\mu = \mathcal{L}([\Xb,\tilde{\Xb}], Y)$ we will still denote $\mu_{swap(\sigma)} := \mathcal{L}([\Xb,\tilde{\Xb}]_{swap(\sigma)}, Y)$.

\paragraph{Local Feature Statistics}
The next step is to extend the construction of feature statistics to the local setting.
\begin{definition}
Define local importance scores as a mapping:
\begin{equation} \Phi : 
    \begin{cases}
    Pr(\mathbb{R}^d \times \mathbb{R}^d\times \mathbb{R}) \longrightarrow \big(\mathbb{R}^d \rightarrow \mathbb{R}^d\times \mathbb{R}^d\big) 
    \\ \mu \mapsto (\Tb_{\mu}, \tilde{\Tb}_{\mu})
    \end{cases}
\end{equation}
where 
\begin{equation}
    (\Tb_{\mu}, \tilde{\Tb}_{\mu}) : \begin{cases}
    \mathbb{R}^d \rightarrow \mathbb{R}^d\times \mathbb{R}^d 
    \\ \zb \mapsto (\Tb_{\mu}(\zb), \tilde{\Tb}_{\mu}(\zb))
    \end{cases}
\end{equation}
such that, for any $S\subset [d]$, we have 
\[
\Phi(\mu_{swap(S)}) = [\Phi(\mu)]_{swap(S)}
\]

For $r>0$, we say that such importance scores $\Phi$ are $r$-local if, for any $\mu \in  Pr(\mathbb{R}^d \times \mathbb{R}^d\times \mathbb{R})$, we have that $\Phi(\mu)(\zb) = (\Tb_{\mu}(\zb), \tilde{\Tb}_{\mu}(\zb))$ only depends on $\mu$ through the restriction of $\mu$ to $B(\zb, r) \times B(\zb, r) \times \mathbb{R}$. That is, if $\mu, \mu'$ are two probability measures on $\mathbb{R}^d \times \mathbb{R}^d\times \mathbb{R}$ such that they coincide on $B(\zb, r) \times B(\zb, r) \times \mathbb{R}$, then  $(\Tb_{\mu}(\zb), \tilde{\Tb}_{\mu}(\zb)) =  (\Tb_{\mu'}(\zb), \tilde{\Tb}_{\mu'}(\zb))$.
\end{definition}

The next goal consists in translating the swap operation in $\mu_{swap(\sigma)} = \mathcal{L}([\Xb,\tilde{\Xb}]_{swap(\sigma)}, Y)$ into a swap of $[\Tb_{\mu}, \tilde{\Tb}_{\mu}]_{swap(\sigma)}$. This step does not require $\tilde{\Xb}$ to be a knockoff of $\Xb$: in what follows we do not make any assumption on $\mu$. Notice that the swap operation has been defined (Definition ~\ref{def-knockoff-swap}) as a transformation of a mapping $\mathbb{R}^d\times \mathbb{R}^d\rightarrow \mathbb{R}^d\times \mathbb{R}^d $, but it can be immediately extended to mappings $\mathbb{R}^d \rightarrow \mathbb{R}^d\times \mathbb{R}^d$. We are able to relate $[\Tb_{\mu_{swap(\sigma)}}, \tilde{\Tb}_{\mu_{swap(\sigma)}}]$ and $[\Tb_{\mu}, \tilde{\Tb}_{\mu}]_{swap(\sigma)}$ if we assume locality constraints on the importance scores.

\begin{definition}
For $\sigma : \mathbb{R}^d \rightarrow \mathcal{P}([d])$ local swap, and $r>0$, define
\begin{align*}
A^r & = \{\zb \in \mathbb{R}^d, \; \sigma(\zb) = \sigma(\yb) \; \forall \yb \in B(\zb,r)\}
\\ \sigma^r & : \zb \mapsto \begin{cases}
\sigma(\zb) \quad \text{if} \;\; \zb \in A^r
\\ \emptyset \quad \text{else}
\end{cases}
\end{align*}
\end{definition}

\begin{proposition}
Assume there exists $r>0$ such that the importance scores are $r$-local. Then, assuming $A^{r}$ is non-empty, for $\zb \in A^{r}$, we have:
\[
[(\Tb_{\mu}(\zb), \tilde{\Tb}_{\mu}(\zb))]_{swap(\sigma(\!\zb\!))}\! =\! [\Tb_{\mu_{swap(\sigma)}}\!(\zb), \tilde{\Tb}_{\mu_{swap(\sigma)}}\!(\zb)]
\]
\end{proposition}

We prove this result in Section~\ref{proof:swap-scores}. Whenever $\mu = \mathcal{L}([\Xb, \tilde{\Xb}], Y)$ where $\tilde{\Xb}$ is a knockoff of $\Xb$, consider a local swap $\sigma$ such that  $\sigma \subset \mathcal{H}_0^0$. By proposition~\ref{extended-exchangeability}, we get that $\mu = \mu_{swap(\sigma)}$, and therefore $(\Tb_{\mu}, \tilde{\Tb}_{\mu}) = (\Tb_{\mu_{swap(\sigma)}}, \tilde{\Tb}_{\mu_{swap(\sigma)}})$. In practice, we can imagine that instead of using $\mu$ as an input when constructing importance scores, we use $\hat{\mu}_n$, an empirical measure defined by the dataset of $n$ i.i.d. samples from $\mu$ that we feed as input to the algorithm. Therefore a consequence of proposition~\ref{extended-exchangeability} will be that, taking also into account the eventual randomness when generating importance scores, we have for any $\zb \in \mathbb{R}^d$,
\[
(\Tb_{\hat{\mu}_{n}}(\zb), \tilde{\Tb}_{\hat{\mu}_n}(\zb)) \stackrel{d}{=} (\Tb_{\hat{\mu}_{n,swap(\sigma)}}(\zb), \tilde{\Tb}_{\hat{\mu}_{n,swap(\sigma)}}(\zb))
\]

We now combine this equality with that of Proposition~\ref{swap-scores}. Fix $r>0$ such that we have $r$-local importance scores. Consider a set of $L$ points $\zb_1, \dots \zb_N \in \mathbb{R}^d$ that are pairwise $2r$ far apart, that is, for any $1\!\leq \!l,l'\!\leq\! N$, \mbox{$\Vert \zb_l - \zb_{l'} \Vert_{\infty} \geq 2r$}.

\begin{proposition}
$\sigma^r$ is a local swap and if $\sigma \subset \mathcal{H}_0^0$, then $\sigma^r \subset \mathcal{H}_0^r$. Furthermore, 
\begin{align*}
\big[ \Tb_{\hat{\mu}_{n}}(\zb_l), &\tilde{\Tb}_{\hat{\mu}_n}(\zb_l) \big]_{1\leq l \leq L} 
\\ & \stackrel{d}{=} \big[ [\Tb_{\hat{\mu}_{n}}(\zb_l), \tilde{\Tb}_{\hat{\mu}_{n}}(\zb_l)]_{swap(\sigma^r(\zb_l))}\big]_{1\leq l \leq L}
\end{align*}
\end{proposition}

We prove this proposition in Section~\ref{proof-sigma-r}. This allows us to conclude. Fix a target FDR level $q \in (0,1)$. Indeed, Proposition~\ref{sigma-r} directly implies the flip-sign condition of Lemma 3.3 in \cite{KN2}. Independently for each $\zb_l$, consider an independent random variable $\epsilon_l = (\epsilon_{l,1}, \dots, \epsilon_{l,d})$, where for each $1\leq l \leq L$, and $1\leq j \leq d$ we have $\epsilon^l_j = 1$ if $j\notin \sigma^r(\zb_l)$, and a Rademacher random variable if $j \in \sigma^r(\zb_l)$. Then denoting $\sigma^r_{\epsilon}(\zb_l) = \sigma_r(\zb_l) \cap \{j: \epsilon_{l,j} = -1\}$, we have: 

\begin{align*}
(\Tb_{\hat{\mu}_{n}}(\zb_l), &\tilde{\Tb}_{\hat{\mu}_n}(\zb_l) )
\\ & \stackrel{d}{=} [\Tb_{\hat{\mu}_{n}}(\zb_l), \tilde{\Tb}_{\hat{\mu}_{n}}(\zb_l)]_{swap(\sigma_{\epsilon}^r(\zb_l))}
\end{align*}
As a consequence, denoting
\[
\Wb_{\hat{\mu}_n}(\zb_l) = \Tb_{\hat{\mu}_{n}}(\zb_l) - \tilde{\Tb}_{\hat{\mu}_n}(\zb_l)
\]
we get that 
\[
\epsilon_l \odot \Wb_{\hat{\mu}_n}(\zb_l)  \stackrel{d}{=}  \Wb_{\hat{\mu}_n}(\zb_l) 
\]
where the symbol $\odot$ indicates component-wise multiplication. Now, given that \mbox{$\Vert \zb_l - \zb_{l'} \Vert_{\infty} \geq 2r$}, this equality holds uniformly for $1\leq l \leq L$. The random choice of the swap $\sigma^r_{\epsilon}(\zb_l)$ is done independently of a random swap $\sigma^r_{\epsilon}(\zb_{l'})$ at another point $\zb_{l'}$. We conclude that the knockoff selection procedure now applies to each of these vectors in an independent way: that is, for each $1\leq l \leq L$, setting 
\[
\hat{\tau}_l = \min \Big\{ t > 0 : \frac{1+ \#\{j:  [\Wb_{\hat{\mu}_n}(\zb_l)]_j  \leq -t\}}{\#\{j: [\Wb_{\hat{\mu}_n}(\zb_l)]_j\}} \leq q\Big\}
\]
allows to construct selection sets $\Hat{\mathcal{S}}_l = \{j: [\Wb_{\hat{\mu}_n}(\zb_l)]_j \geq \hat{\tau}_l\}$, that control $\text{FDR}^r$ given that initially $\sigma^r(\zb) \subset \mathcal{H}_0^r(\zb)$.

That is, according to Theorem 3.4 in \cite{KN2}, the set $\Hat{\mathcal{S}}_l$ is such that  
\[
\mathbb{E}[\frac{|\hat{\mathcal{S}}_l\cap \mathcal{H}_0^r(\xb)|}{1\vee | \hat{\mathcal{S}}_l|}] \leq q
\]
hence the result.

\section{Proofs}

\subsection{Proof of Proposition~\ref{extended-exchangeability}}\label{proof-extended-exchangeability}

\begin{proof}
We begin the proof with two lemmas:

\begin{lemma}\label{swap-decomposition}
We can decompose a local swap $\sigma : \mathbb{R}^d \rightarrow \mathcal{P}([d])$ into $\sigma_i :\mathbb{R}^d \rightarrow \mathcal{P}([d])$ such that for every $i \in [d]$:
\begin{equation*}
    \begin{cases}
    Im(\sigma_i) = \{\emptyset, \{i\}\} \\
    \sigma(\xb) = \bigsqcup_{i=1}^d \sigma_i(\xb) , \; \forall \xb \in \mathbb{R}^d \\
    \sigma_i \quad \text{is a local swap}
    \end{cases}
\end{equation*}
We will denote by $\sigma = \bigsqcup_{i=1}^d \sigma_i$ the property $\sigma(\xb) = \bigsqcup_{i=1}^d \sigma_i(\xb) , \; \forall \xb \in \mathbb{R}^d$.
\end{lemma}

\begin{proof}
Define, for every $i \in [d]$, 
\begin{equation*}
\sigma_i(x) = 
\begin{cases} 
\,\emptyset  & \text{if} \: i\notin \sigma(x) \\
\{i\}    &\text{if} \: i \in \sigma(x)
\end{cases}
\end{equation*}

We need to show that $\sigma_i$ is a local swap. Let $\xb, \zb \in \mathbb{R}^d$ such that $\xb_{[d]\setminus \sigma_i(\xb)} = \zb_{[d]\setminus \sigma_i(\xb)}$. If $\sigma_i(\xb) = \emptyset$ then $\xb = \zb$ and therefore $\sigma_i(\xb) = \sigma_i(\zb)$. If $\sigma_i(\xb) = \{i\}$, then $i \in \sigma(\xb)$, so  $\xb_{[d]\setminus \sigma_i(\xb)} = \zb_{[d]\setminus \sigma_i(\xb)}$ implies $\xb_{[d]\setminus \sigma(\xb)} = \zb_{[d]\setminus \sigma(\xb)}$ and therefore $\sigma(\xb) = \sigma(\zb)$ given that $\sigma$ is a local swap. But then we have $\sigma_i(\xb) = \sigma_i(\zb)$ by definition of $\sigma_i$.
\end{proof}

\begin{lemma}\label{swap-iteration}
Assume that we have a partition of a local swap $\sigma$ into two local swaps $\sigma^a, \sigma^b$: $\sigma = \sigma^a \bigsqcup \sigma^b$. Then we have :

\[
[\Fb^{Id}\!, \tilde{\Fb}^{Id}]_{swap(\!\sigma\!)} \! = \! [\Fb^{Id}\!, \tilde{\Fb}^{Id}]_{swap(\!\sigma^b\!)} \circ [\Fb^{Id}\!, \tilde{\Fb}^{Id}]_{swap(\!\sigma^a\!)}
\]
\end{lemma}
\begin{proof}
It is crucial to notice that, given our previous swap definition for an identity mapping, the $swap(\sigma^b)$ operator applies to the output of the swapped vector by $\sigma^a$. In order to prove the result we need to show that $\sigma^b(\xb) =\sigma^b(\zb)$ where $\zb$ is the vector of the first $d$ coordinates of $[\xb,\tilde{\xb}]_{swap(\sigma^a(\xb))}$. Given that, for any $\xb \in \mathbb{R}^d$ we have $\xb_{[d]\setminus \sigma^a(\xb)} = \zb_{[d]\setminus \sigma^a(\xb)}$ and $\sigma^a(\xb) \subset \sigma(\xb)$, we get $\sigma(\xb) = \sigma(\zb)$ and same equality with $\sigma^a$, as both $\sigma, \sigma^a$ are local swaps. That implies $\sigma^b(\xb) = \sigma^b(\zb)$ as $\sigma = \sigma^a \bigsqcup \sigma^b$. Notice that the order does not matter when composing the two swapped identity mappings.
\end{proof}

In order to prove equation~\ref{general-exchangeability}, we write it in terms of mappings: we want to show that if $(\Xb, \tilde{\Xb})$ satisfy exchangeability, then
\begin{align*}
[\Xb,\tilde{\Xb}]_{swap(\sigma)} & = [\Fb^{Id},  \tilde{\Fb}^{Id}]_{swap(\sigma)}(\Xb,\tilde{\Xb}) 
\\ &\stackrel{d}{=} [\Fb^{Id}, \tilde{\Fb}^{Id}](\Xb,\tilde{\Xb}) = [\Xb, \tilde{\Xb}]
\end{align*}
With Lemma~\ref{swap-decomposition} we decompose $\sigma = \bigsqcup_{i=1}^d \sigma_i$, and by recursively using Lemma~\ref{swap-iteration} we get that 
\begin{align*}
[\Fb^{Id}, \tilde{\Fb}^{Id}]_{swap(\sigma)}  = [\Fb^{Id}, \tilde{\Fb}^{Id}]_{swap(\sigma_1)} &\circ \dots 
\\ \dots \circ  [\Fb^{Id},& \tilde{\Fb}^{Id}]_{swap(\sigma_d)}
\end{align*}
It then suffices to show the equality in distribution for just one swap operation, so that we can recursively apply the swapped identity mappings while keeping the equality in distribution. We then need to prove that :
\[
[\Fb^{Id}, \tilde{\Fb}^{Id}]_{swap(\sigma_1)}(\Xb,\tilde{\Xb})  \stackrel{d}{=} [\Fb^{Id}, \tilde{\Fb}^{Id}](\Xb,\tilde{\Xb}) 
\]
Equivalently, if we condition on $\Xb_{-1}, \tilde{\Xb}_{-1}$ we need to show that
\begin{align*}
[\Fb^{Id}, \tilde{\Fb}^{Id}]_{swap(\sigma_1)}(\Xb&,\tilde{\Xb})|\Xb_{-1},\tilde{\Xb}_{-1}
\\ &\stackrel{d}{=} [\Fb^{Id}, \tilde{\Fb}^{Id}](\Xb,\tilde{\Xb})  |\Xb_{-1},\tilde{\Xb}_{-1} 
\end{align*}
Here crucially we use the fact that $\sigma_1$ is a local swap. Indeed, whenever we condition on $\Xb_{-1}, \tilde{\Xb}_{-1}$, the input values to the mapping $[\Fb^{Id}, \tilde{\Fb}^{Id}]_{swap(\sigma_1)}$ can be seen as constant with respect to $\Xb_{-1}, \tilde{\Xb}_{-1}$. Given that $\sigma_1$ can only be equal to $\emptyset$ or $\{1\}$, and that therefore its value is determined by $\Xb_{-1}$, hence constant when we condition on $\Xb_{-1}$, we get that either 
\begin{align*}
[\Fb^{Id}, \tilde{\Fb}^{Id}]_{swap(\sigma_1)}(\Xb &,\tilde{\Xb})|\Xb_{-1},\tilde{\Xb}_{-1} 
\\ & = [\Fb^{Id}, \tilde{\Fb}^{Id}](\Xb,\tilde{\Xb}) |\Xb_{-1},\tilde{\Xb}_{-1} 
\end{align*}
which therefore holds also in distribution or 
\begin{align*}
[\Fb^{Id}, \tilde{\Fb}^{Id}&]_{swap(\sigma_1)}(\Xb  ,\tilde{\Xb})|\Xb_{-1},\tilde{\Xb}_{-1} \\
 &
 = [\Fb^{Id}, \tilde{\Fb}^{Id}]_{swap(\{1\})}(\Xb,\tilde{\Xb})|\Xb_{-1},\tilde{\Xb}_{-1}
\end{align*}
In that case, it simplifies into 
\[
X_1,\tilde{X}_1 |\Xb_{-1},\tilde{\Xb}_{-1} \stackrel{d}{=} \tilde{X}_1, X_1 |\Xb_{-1},\tilde{\Xb}_{-1} 
\]
which is a consequence of the fact that $\Xb, \tilde{\Xb}$ satisfy exchangeability, hence the result.

To prove \mbox{$[\Xb,\tilde{\Xb}]_{swap(\sigma)}, Y \stackrel{d}{=} [\Xb,\tilde{\Xb}], Y$}, we assume that for every $\xb \in \mathbb{R}^d$, $\sigma(\xb) \subset \mathcal{H}_0^0(\xb)$. Jointly taking $Y$ with $(\Xb, \tilde{\Xb})$, the proof is the same up to proving that the following holds whenever $1 \in \mathcal{H}_0^0(\Xb)$: 
\[
X_1,\tilde{X}_1, Y |\Xb_{-1},\tilde{\Xb}_{-1} \stackrel{d}{=} \tilde{X}_1, X_1, Y |\Xb_{-1},\tilde{\Xb}_{-1} 
\]
By the properties of $\mathcal{H}_0^0$, $1 \in \mathcal{H}_0^0(\Xb)$ holds regardless of the value of $X_1$ when conditioning on $\Xb_{-1}$. Now if we write down the densities (as we assumed that the joint distribution has a positive density with respect to a product measure):
\begin{align*}
    p(\xb, \tilde{\xb}&,y) = p(y|\xb, \tilde{\xb})p(\xb, \tilde{\xb})
    \\ &= p(y|\xb) p(\xb, \tilde{\xb}) \quad \text{as}\quad \tilde{\Xb} \independent Y |\Xb
    \\ & =p(y|\xb_{-1}) p(\xb, \tilde{\xb}) \quad \text{as} \quad X_1 \independent Y |\Xb_{-1} = \xb_{-1}
    \\ & = p(y|\xb_{-1}) p([\xb, \tilde{\xb}]_{swap(\!\{\!1\!\}\!)}) \quad \text{by exchangeability}
    \\ & = p([\xb, \tilde{\xb}]_{swap(\{1\})},y)
\end{align*}

Hence the result.
\end{proof}

\subsection{Proof of Proposition~\ref{swap-scores}}\label{proof:swap-scores}

\begin{proof}
Fix $\sigma$ local swap, $r>0$, and $\zb \in A^{r}$. Let $S := \sigma(\zb)$, by definition of local importance scores we have that 
\[
[(\Tb_{\mu}(\cdot), \tilde{\Tb}_{\mu}(\cdot))]_{swap(S)} = [(\Tb_{\mu_{swap(S)}}(\cdot), \tilde{\Tb}_{\mu_{swap(S)}}(\cdot))]
\]
By definition of $r$-local importance scores, the value of $[\Tb_{\mu_{swap(\sigma)}}(\zb), \tilde{\Tb}_{\mu_{swap(\sigma)}}(\zb)]$ depends on $\mu_{swap(\sigma)}$ only through $B(\zb,r) \times B(\zb,r)\times \mathbb{R}$. Therefore if $\mu_{swap(\sigma)}$ and $\mu_{swap(S)}$ coincide on $B(\zb,r) \times B(\zb,r)\times \mathbb{R}$, then we have
\begin{align*}
[\Tb_{\mu_{swap(\sigma)}}(\zb), & \tilde{\Tb}_{\mu_{swap(\sigma)}}(\zb)] 
\\ & = [\Tb_{\mu_{swap(S)}}(\zb), \tilde{\Tb}_{\mu_{swap(S)}}(\zb)]
\\ & = [(\Tb_{\mu}(\zb), \tilde{\Tb}_{\mu}(\zb))]_{swap(\sigma(\zb))}
\end{align*}
Given that $\zb \in A^{r}$, we have that $\forall \ub, \vb \in B(\zb, r)$, $\sigma(\ub) =  S$. We want to show that, for any $y \in \mathbb{R}$, 
\begin{align*}
\mu_{swap(\sigma)}&(\ub, \vb,y)  = \mu_{swap(S)}(\ub,\vb,y)
\\ \Leftrightarrow \mathbb{P}\big( [\Xb,&\tilde{\Xb}]_{swap(\sigma)},Y = \ub,\vb,y\big) 
\\ & =  \mathbb{P}\big( [\Xb,\tilde{\Xb}]_{swap(S)},Y = \ub,\vb,y \big)
\\ \Leftrightarrow \mathbb{P}\big( [\Fb^{Id},&\tilde{\Fb}^{Id}]_{swap(\sigma)}(\Xb,\tilde{\Xb}),Y = \ub,\vb,y\big) 
\\ & = \! \mathbb{P}\big( [\Fb^{Id},\tilde{\Fb}^{Id}]_{swap(S)}(\Xb,\tilde{\Xb}),Y = \ub,\vb,y \big)
\\  \Leftrightarrow \mathbb{P}\big( \Xb,&\tilde{\Xb},Y = [\Fb^{Id},\tilde{\Fb}^{Id}]_{swap(\sigma)}(\ub,\vb),y\big) 
\\ & = \! \mathbb{P}\big( \Xb,\tilde{\Xb},Y = [\Fb^{Id},\tilde{\Fb}^{Id}]_{swap(S)}(\ub,\vb),y \big)
\\  \Leftrightarrow \mathbb{P}\big( \Xb &,\tilde{\Xb},Y = [\ub,\vb]_{swap(\sigma(\ub))},y\big) 
\\ & =  \mathbb{P}\big( \Xb,\tilde{\Xb},Y = [\ub,\vb]_{swap(S)},y \big)
\end{align*}
Hence the result.
\end{proof}

\subsection{Extended Flip-Sign Property for Local Swaps and Local Importance Scores}\label{proof-sigma-r}
\begin{proposition}\label{sigma-r}
$\sigma^r$ is a local swap and if $\sigma \subset \mathcal{H}_0^0$, then $\sigma^r \subset \mathcal{H}_0^r$. Furthermore, 
\begin{align*}
\big[ \Tb_{\hat{\mu}_{n}}(\zb_l), &\tilde{\Tb}_{\hat{\mu}_n}(\zb_l) \big]_{1\leq l \leq L} 
\\ & \stackrel{d}{=} \big[ [\Tb_{\hat{\mu}_{n}}(\zb_l), \tilde{\Tb}_{\hat{\mu}_{n}}(\zb_l)]_{swap(\sigma^r(\zb_l))}\big]_{1\leq l \leq L}
\end{align*}
\end{proposition}

\begin{proof}
Let $\xb, \zb$ such that $\xb_{[d]\setminus \sigma^r(\xb)} = \zb_{[d]\setminus \sigma^r(\xb)}$. We want to show that $\sigma^r(\xb) = \sigma^r(\zb)$. 

If $\xb \notin A^r$, then $\sigma^r(\xb) = \emptyset$, so $\xb = \zb$ and $\sigma^r(\xb) = \sigma^r(\zb)$. 

If $\xb \in A^r$, let us first show that it implies $\zb \in A^r$. Let $\yb \in B(\zb, r)$, show that $\sigma(\yb) = \sigma(\zb)$. We have that $\yb-(\zb - \xb) \in B(\xb, r)$, and given that $\xb \in A^r$, we get $\sigma(\xb) = \sigma(\yb-(\zb - \xb))$. As $\sigma$ is a local swap, we have $\sigma(\xb) = \sigma(\zb)$, and we also get $\sigma(\yb-(\zb - \xb)) = \sigma(\yb)$ because
\begin{align*}
[\yb-(\zb - \xb)]_{[d]\setminus \sigma(\yb-(\zb - \xb))} &= \yb_{[d]\setminus \sigma(\yb-(\zb - \xb))}
\\ \Leftrightarrow \quad [\yb-(\zb - \xb)]_{[d]\setminus \sigma(\xb)} &= \yb_{[d]\setminus \sigma(\xb)}
\\ \Leftrightarrow \quad [(\zb - \xb)]_{[d]\setminus \sigma(\xb)} &= 0
\\ \Leftrightarrow \quad \xb_{[d]\setminus \sigma^r(\xb)} &= \zb_{[d]\setminus \sigma^r(\xb)}
\end{align*}
We can now conclude: $\sigma^r(\xb) = \sigma(\xb)$ as $\xb \in A^r$, and given that $\sigma$ is a local swap we get that $\sigma(\xb) = \sigma(\zb)$. Finally, as $\zb \in A^r$, we have $\sigma(\zb) = \sigma^r(\zb)$ and therefore $\sigma^r(\xb) = \sigma^r(\zb)$.

Assume $\sigma \subset \mathcal{H}_0^0$. Fix $\zb \in \mathbb{R}^d$, assume that $\sigma^r(\zb) \neq \emptyset$ and take $j \in \sigma^r(\zb)$. That implies $\zb \in A^r$, and therefore $\forall \yb \in B(\zb,r)$ we have $j \in \sigma^r(\zb) = \sigma(\zb) = \sigma(\yb) \subset \mathcal{H}_0^0(\yb)$. Therefore $j \in \cap_{\yb \in B(\zb, r)} \mathcal{H}_0^0(\yb) = \mathcal{H}_0^r(\zb)$. We then conclude that $\sigma^r \subset \mathcal{H}_0^r$.

The last statement is a concatenation of Proposition~\ref{swap-scores} and the fact that $\mu = \mu_{swap(\sigma^r)}$.
\end{proof} 

\section{Semi-synthetic Data Experiments}

Our simulations previously described were entirely based on synthetic data. Alternatively, using real SNPs data and then fitting a HMM model yields the same experimental results, which we did for data from the 1000 Genomes Project \citep{1000genomes}, where we obtained around 2000 individual samples for 27 distinct segments of chromosome 19 containing an average of 50 SNPs per segment, and filtered out SNPs that are extremely correlated (above 0.95). This is because HMMs can truly capture the covariate distribution of a SNP dataset and are a good model for a downstream feature selection with the knockoff procedure. For simplicity, and in order to scale with the number of samples (which is limited with real data), we described simulations based on synthetic covariates.

\section{Saliency-based Partitioning}\label{appendix:partition}

Saliency maps \cite{lipton2016mythos} have emerged as a popular tool for interpretability in Neural Networks. A saliency map allows to identify, under a trained model that minimizes some loss $L$, which input variation has the strongest impact on the loss at a given training point. Training a neural network on the concatenated vector $\Xb, \tilde{\Xb}$ to predict $Y$, saliency maps can be used as importance scores at any given training point. That is, denoting $g_{\theta}: \mathbb{R}^d \times \mathbb{R}^d \rightarrow \mathbb{R}$ a classifier parametrized by $\theta$ (such as a neural network), consider $\hat{\theta}_n$ the output of training such model on the actual data. We can now compute the saliency scores:
\begin{align*}
    \mathbb{S}, \tilde{\mathbb{S}}:
    \begin{cases}
    \mathbb{R}^d \times \mathbb{R}^d \rightarrow \mathbb{R}^d \times \mathbb{R}^d
    \\ \Xb_i , \tilde{\Xb}_i \mapsto \nabla_{\Xb, \tilde{\Xb}}L(g_{\hat{\theta}_n}(\Xb_i, \Hat{\Xb_i})-Y_i)
    \end{cases}
\end{align*}
Notice that the saliency scores are only computed for training points $\Xb_i , \tilde{\Xb}_i, Y_i$, but our definition of $\mathbb{S},\tilde{\mathbb{S}}$ assumes that we can expand the saliency scores to the whole feature space (for example, through smoothing). The reason why these mappings can not be immediately used as local importance scores is because of the training process: the output of the training process are the parameter estimates in $\Hat{\theta}_n = \Hat{\theta}_n(\Xb, \tilde{\Xb}, Y)$, which are constructed based on the global training set. Even though the saliency is local at a point the swap operation can not go through the training process, i.e. we can not relate $g_{\Hat{\theta}_n(\Xb, \tilde{\Xb},Y)}([\xb, \tilde{\xb}]_{swap(\sigma)})$ and $g_{\Hat{\theta}_n([\Xb, \tilde{\Xb}]_{swap(\sigma)},Y)}([\xb, \tilde{\xb}])$. This is due to the influence that a training point lying in one region of the space may have at another point (during evaluation) at a different region \cite{koh2017understanding}. Still, we can use these saliency scores to partition the space based on a subsample of the whole initial dataset, and then run the Knockoff procedure with local importance scores at points located in each subregion. This method has the advantage of being computationally less expensive than the previous one, especially in high dimensions.

\end{document}